%% file: main.tex
\documentclass{article}

\PassOptionsToPackage{numbers,sort&compress}{natbib}
\usepackage[main,final]{neurips_2026}

\usepackage[utf8]{inputenc} 
\usepackage[T1]{fontenc}    
\usepackage{hyperref}       
\hypersetup{hidelinks}
\usepackage{url}            
\usepackage{booktabs}       
\usepackage{amsfonts}       
\usepackage{nicefrac}       
\usepackage{microtype}      
\usepackage{xcolor}         

\input{additional_preamble.tex}

\input{macros.tex}

\title{Auction-Based Online Policy Adaptation\\for Evolving Objectives}

\author{%
  Guruprerana Shabadi \\
  University of Pennsylvania \\
  United States \\
  \texttt{shabadi@engineering.upenn.edu}
  \And
  Kaushik Mallik \\
  IMDEA Software Institute \\
  Spain \\
  \texttt{kaushik.mallik@imdea.org}
}

\begin{document}

\maketitle

\input{abstract}
\input{intro}
\input{related-work}
\input{problem-setup}

\input{learning.tex}
\input{implementation-eval}

\input{conclusion}

\FloatBarrier
\begin{ack}
Guruprerana Shabadi is supported by NSF Award SLES 2331783.
Kaushik Mallik is supported by the grant RYC2024-049116-I funded by MICIU/AEI/10.13039/501100011033 and the ESF+.
\end{ack}

\bibliographystyle{plainnat}
\bibliography{main}


\clearpage
\appendix

\input{deterministic-appendix}

\input{poorman-proof-appendix}
\input{allpay-proof-appendix}
\input{envs-appendix}
\input{additional-experimental-results}
\input{implementation-details-appendix}
\input{hyperparameters}

\clearpage
\input{checklist.tex}

\end{document}

%% file: additional_preamble.tex
\usepackage{amsthm,amsmath,amssymb}
\usepackage{cleveref}
\usepackage{xspace}
\usepackage{wrapfig}            
\usepackage{placeins}           
\usepackage{booktabs}           
\usepackage{graphicx}
\usepackage{adjustbox}
\usepackage{mathtools}
\usepackage{subcaption}
\usepackage{xfrac}
\usepackage{tikz}
\usetikzlibrary{arrows.meta,calc,positioning}

\theoremstyle{plain}
\newtheorem{theorem}{Theorem}
\newtheorem{lemma}{Lemma}

\newtheorem{proposition}{Proposition}

\theoremstyle{definition}
\newtheorem{definition}{Definition}

\theoremstyle{remark}

%% file: macros.tex
\newcommand{\dist}{\Delta}
\newcommand{\N}{\mathbb{N}}

\newcommand{\R}{\mathbb{R}}

\newcommand{\abs}[1]{\left|#1\right|}

\newcommand{\momdp}{MO-MDP\xspace}
\newcommand{\momdps}{MO-MDPs\xspace}

\renewcommand{\vec}[1]{\mathbf{#1}}

\newcommand{\M}{\mathcal{M}}

\newcommand{\poorman}{\textsf{Winner-Pays}\xspace}
\newcommand{\allpay}{\textsf{All-Pay}\xspace}

\newcommand{\E}{\mathbb{E}}

\newcommand{\slack}{\mathit{slack}}

\newcommand{\win}{\mathsf{win}}

\newcommand{\lose}{\mathsf{lose}}
\newcommand{\cont}{\mathsf{cont}}



%% file: abstract.tex
\begin{abstract}
We consider multi-objective reinforcement learning problems where objectives come from an identical family---such as the class of reachability objectives---and may appear or disappear at runtime.
Our goal is to design adaptive policies that can efficiently adjust their behaviors as the set of active objectives changes.
To solve this problem, we propose a modular framework where each objective is supported by a selfish local policy, and coordination is achieved through a novel \emph{auction}-based mechanism:
policies bid for the right to execute their actions, with bids reflecting the urgency of the current state.
The highest bidder selects the action, enabling a dynamic and interpretable trade-off among objectives.
Going back to the original adaptation problem, when objectives change, the system adapts by simply adding or removing the corresponding policies.
Moreover, as objectives arise from the same family, identical copies of a parameterized policy can be deployed, facilitating immediate adaptation at runtime.
We show how the selfish local policies can be computed by turning the problem into a general-sum Markov game, where the policies compete against each other to fulfill their own objectives.
To succeed, each policy must not only optimize its own objective, but also reason about the presence of other goals and learn to produce calibrated bids that reflect relative priority.
Under mild assumptions, we prove the existence of Nash equilibria where dishonest bidding leads to suboptimal outcome, and the most urgent objectives win control automatically.
In our implementation, the policies are trained concurrently using proximal policy optimization (PPO).
We evaluate on two Atari games and a gridworld-based path-planning task with dynamic targets.
Our method achieves substantially better performance than monolithic policies trained with PPO. 
\end{abstract}

%% file: intro.tex
\section{Introduction}
\label{sec:introduction}

\newcommand{\introfigleftvoffset}{-3cm}

\begin{wrapfigure}[12]{r}{0.3\textwidth}
  \hspace*{0.6em}\begin{minipage}{\dimexpr\linewidth-1em}
    \includegraphics[width=\linewidth]{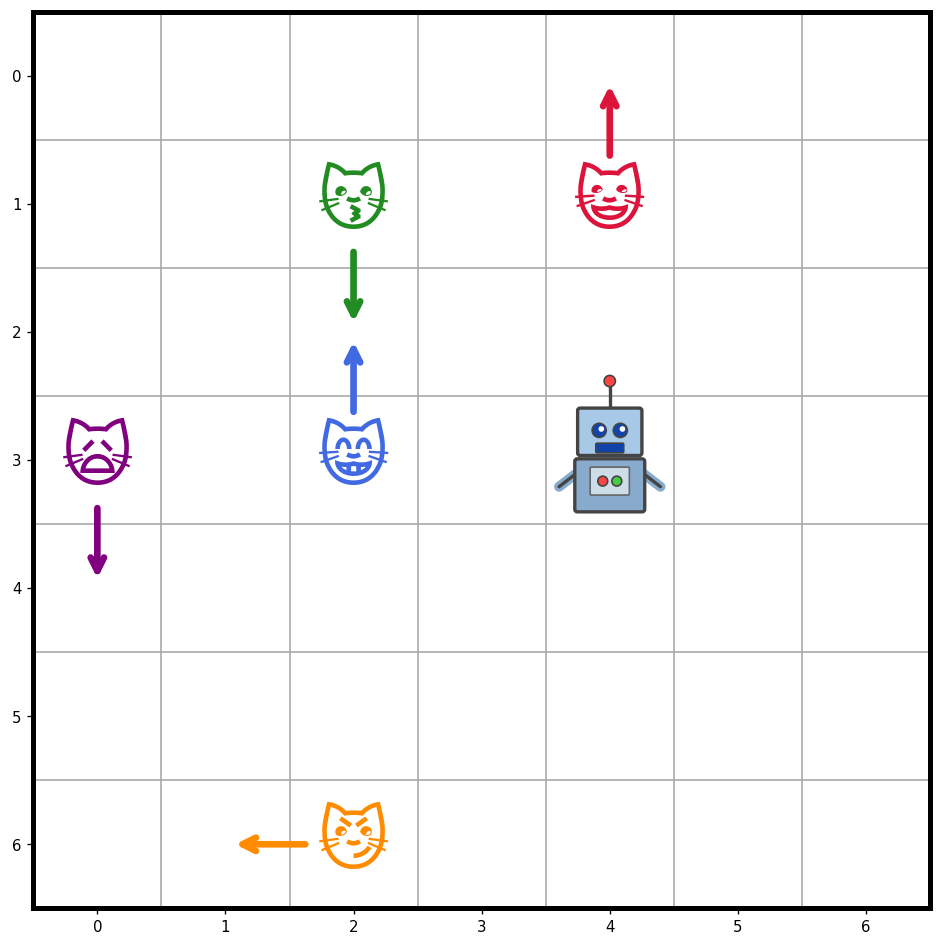}
    \caption{Cat Feeder env.}
    \label{fig:grid_preview_intro}
  \end{minipage}
\end{wrapfigure}
Consider the problem of controlling a mobile robot deployed in a cat shelter whose task is to deliver food to cats roaming around the facility (Figure~\ref{fig:grid_preview_intro}). 
New food requests from hungry cats may arrive at arbitrary times and locations, while existing requests may disappear if a cat wanders away before being served.
Since the robot can serve only one cat at a time, fulfilling one request may delay others, potentially allowing some cats to move farther away. 
Thus, the robot must dynamically adapt to the set of active requests to maximize the number of cats fed before they disappear.

In this work, we model the above control problem as a multi-objective reinforcement learning task in which all the objectives come from an identical family, such as reachability in this case.
These objectives can be conflicting and any number of them may be active at a given time.
Moreover, we make no assumptions about their distribution or arrival pattern.
Our goal is to design adaptive policies that update their behavior seamlessly as objectives appear or disappear.

The class of problems we study is relevant to many real-world scenarios.
For example, a warehouse robot must be responsive to changing pickup, delivery, and sorting tasks, and this setting has motivated work on warehouse robot path planning and resource allocation~\citep{chen2023warehouserobot,shen2023amazonrobotics}.
More broadly, such problems have been extensively studied in the planning literature~\citep{fox2006plan}, typically assuming that a precise environment model (like a graph) is available without uncertainties. 
Existing reinforcement learning-based path planning methods~\citep{singh2023reviewrlpathplanning} are able to handle uncertainties but cannot adapt to varying number of objectives at deployment time.

To solve this problem, we propose a novel compositional RL framework, where each objective is served using a local policy, selfishly maximizing its assigned objective by competing against the other policies.
Coordination between these policies is achieved through a novel auction-based mechanism:
at regular intervals of a given length $\tau>0$, policies submit bids for the right to execute their actions, where bids reflect the urgency of the current state with respect to their objectives. 
The highest bidder selects the action for the next $\tau$ steps, enabling a dynamic and interpretable trade-off among objectives.
Returning to the cat feeder example, suppose the robot approaches an intersection with two active requests, the policy for the first request recommends turning left, while the policy for the second recommends turning right. 
Since both actions cannot be executed simultaneously, a trade-off is unavoidable. 
The auction mechanism determines which objective is prioritized, based on the relative urgency encoded in the bids.

This compositional design yields modularity, facilitating fast adaptation:
when an objective appears or disappears, only the corresponding policy needs to be added or removed.
Furthermore, when the objectives belong to the identical family---which is reachability in the cat feeder example---a single generalized policy can be trained for the entire family. 
We can deploy identical \emph{copies} of this policy when new objectives arrive, enabling immediate adaptation. 
Modularity also permits us to define objective-specific reward signals while the auction mechanism takes care of the trade-offs between objectives.
For example, we can easily define a distance-based dense reward signal for each policy.
In contrast, in the monolithic case, it is challenging to define such a dense reward signal that strikes the ``optimal'' trade-off between all objectives.
We show in our experiments that monolithic policies, trained with both sparse rewards and different types of scalarized dense rewards, achieves significantly lower payoffs as compared to our modular approach.

Yet another advantage of our modular framework is interpretability.
At any moment, we can observe the policy in control, and thereby the objective currently being pursued by the agent. 
This explanation is beneficial in a range of real-world situations, like providing post-hoc explanations of failures, and also feedback on biases in pursuing different objectives.

We implement our modular framework as a general-sum Markov game between local policies seeking to fulfill their assigned objectives, where each policy's action space is augmented with an additional numeric variable representing the bids.
However, new challenges arise, and below we summarize how we overcome them.

\textbf{Challenge I: Enforcing honest bids.}
The success of the auction mechanism hinges on the policies faithfully selecting their bid values according to their true urgency. But what if policies learn to overbid to suppress competitors?
This would lead to suboptimal results. 
To prevent this, policies are penalized in proportion to their bid values.
As a consequence, excessive bidding reduces the net return of a policy, which is the total accumulated rewards minus the total accumulated bidding penalties. 
Under mild assumptions, we prove that our penalty model implements truthful bidding, i.e., policies would bid high only if the earned rewards outweigh the penalties.

\textbf{Challenge II: Achieving environment awareness.}
Effective bidding not only depends on the current state and the local objective, but also on the  objectives of the competing policies. 
For instance, in the cat feeder problem, if two requests are spatially close, their policies need not compete aggressively. 
Conversely, if requests lie in opposite directions, the urgent policy must bid sufficiently high to win. 
To enable such environment-aware bidding, we allow the policies to observe all the objectives, i.e., the policies have information about the global state. 
But a large number of objectives makes the global state vector very big, resulting in sluggish training. 
We equip policies with an \textit{attention pooling} architecture that can extract relevant information from the active objectives into a succinct embedding vector, so that it can be fed as input to the policy.
Each local policy is equipped with its own pooling module which is co-trained with the policy.


\textbf{Challenge III: Variable objective count.}
Recall, we set out to support environments where objectives appear and disappear freely at runtime.
This feature is automatically incorporated by the attention pooling module, which is a learned component with the ability to generalize against different number of competing objectives than the number seen at training time.
Regardless of the objective count, the learned policy always receives a fixed-size encoding of the global state of the environment.


We test our framework on three environments with dynamic expected total payoff objectives: the cat feeder environment implemented as a gridworld, and two Atari games: Assault and Air-raid~\citep{bellemare13arcade}.
We use proximal policy optimization (PPO; \cite{schulman2017ppo}) to train each individual policy competing against other policies in the Markov game encoding of the problem.
We report the empirical convergence of the training process, and show that policies trained using our framework exhibit dramatic improvement over the baselines in terms of performance.
We also demonstrate the ability of our approach to generalize to more objectives at test-time than during training.
Lastly, to aid interpretability in renderings of rollouts, we implement visualization of the policy in control and the objective being pursued, through which we get real-time feedback on the strategic behaviors of the policies (videos are included in the Supplementary Material).


%% file: related-work.tex
\textbf{Related Work.} Most multi-objective reinforcement learning (MORL) methods reduce multiple rewards to a single decision rule, either through scalarization~\citep{gass1955computational,van2013scalarized}, Pareto-based criteria~\citep{van2014multi,pirotta2015multi}, or fairness objectives~\citep{park2024max,byeon2025multi,siddique2020learning}. These approaches typically learn a monolithic policy for a fixed set of reward components. In contrast, our method keeps objectives modular: each local policy pursues one objective, and the auction mechanism composes their actions at runtime.

The closest distributed MORL approaches are W-learning and Deep W-learning~\citep{humphrys1995w,rosero2024multi}, which train separate policies and select among them using learned urgency-like scores. Other decomposition methods combine action values or votes through fixed aggregation rules~\citep{russell2003q,mendez2019multi}. Our approach differs by using bids as the coordination interface, which lets standard policy-gradient training learn both the local action and the urgency signal without a separate meta-policy.

Our design is also related to hierarchical RL, which uses modular subpolicies under a higher-level manager~\citep{klissarov2025discovering}.
In comparison, our composition is flat and competitive rather than manager-driven.
Another related line of work is using auction mechanisms for model-based multi-objective planning and bidding games~\citep{avni2024auction,lazarus1999combinatorial,avni2019infinite,avni2025bidding}.
In comparison, our setting uses model-free RL and supports a varying number of objectives.

%% file: problem-setup.tex
\section{Auction-Based Multi-Policy Framework}\label{sec:auction-framework}

\paragraph{\momdps, Policies, and Values.}
We use multi-objective Markov decision processes to model uncertain environments with multiple reward objectives.
An \momdp with \(m\in \N_{>0}\) objectives is a tuple \(\mathcal{M}=(S,A,T,\mu,R,H)\), where \(S\) is the \textit{state} space, \(A\) is the \textit{action} space, \(T\colon S\times A\to \dist(S)\) is the \textit{transition} kernel, \(\mu\in \dist(S)\) is the \textit{initial state} distribution, \(R=\{r_i\colon S\times A\to \R\}_{i\in [1;m]}\) is the set of \textit{reward functions}, and $H\in \mathbb{N}_{>0}$ is the \textit{time horizon}.
A (Markov) \textit{policy} \(\pi:S\to \dist(A)\) maps the current state to a distribution over actions. 
For a trace \(\rho=s^0a^0s^1a^1\ldots s^H\), the accumulated \textit{local} reward of objective \(i\) is \(w_i(\rho)=\sum_{t=0}^{H-1} r_i(s^t,a^t)\).
In the environments we consider, the optimization metric is the cumulative reward across objectives and so we define the \textit{global} reward (aka, scalarized reward~\cite{van2013scalarized}) as \(w(\rho)=\sum_{i=1}^m w_i(\rho)\).
The \textit{value} of a policy is \(V^\pi=\E^\pi[w(\rho)]\), and an \textit{optimal} policy is any policy \(\pi^*\in \arg\max_\pi V^\pi\).



\smallskip


We present a novel \textit{auction}-based RL framework for compositional policy synthesis for \momdps.
In our framework, we have objective-specific local policies that not only emit actions, but also submit \emph{bids} within a given bound $\beta\in \mathbb{N}_{>0}$ to win the privilege of executing their actions for a given number of time steps $\tau\in \mathbb{N}_{>0}$.
We consider two auction mechanisms, called \poorman and \allpay.
In~\poorman, the bids are non-negative \textit{integers} in the interval $[0;\beta]$, and in~\allpay, the bids are non-negative \textit{real} numbers in the interval $[0,\beta]$.%
\footnote{The distinction in bid types---\textit{integers} for \poorman and \textit{reals} for \allpay---is purely technical. 
The synthesis of local policies will involve a game-theoretic analysis involving Nash equilibria.
For real-valued bids in \poorman, it is unclear if Nash equilibria would always exist, owing to the continuous bid values and discontinuities in the payoff profile due to the tie-breaking mechanism~\cite{simon1990discontinuous}. 
For integer-valued bids in \allpay, the shape of the Nash equilibria policies is only known for the two-policy case~\cite{dziubinski2023discrete}, but the general case is still open.
We leave these unresolved cases for a future work.
}
In both settings, the highest bidder's actions get executed for $\tau$ consecutive steps, with bidding ties being resolved uniformly at random.
To encourage truthfulness, we penalize the policies proportional to their bids scaled with a constant $\rho\in (0,1)$.
In the \poorman model, only the executed policy (one of the highest bidders) pays the penalty, whereas in \allpay, all policies pay penalties.
The penalties force policies to truthfully select their bids to reflect how critical their action actually is in the current state.
We remark that this framework is a purely decentralized scheme to coordinate local policies in a given \momdp.
The parameters $\tau$, $\beta$, and $\rho$ denote the \emph{bidding interval}, \emph{maximum bid}, and \emph{bid penalty coefficient}, respectively. Ablation studies in Section~\ref{sec:experimental-results} highlight their importance. If $\tau$ is too small, policies switch too frequently, risking suboptimal outcomes. If $\beta$ is too small, policies cannot express preference strength, leading to near-random control. If $\rho$ is too large, policies tend to abstain from participating in the bidding.

We formalize our framework as a general-sum game between multiple policies.

\begin{definition}[Markov bidding game]\label{def:markov-bidding-game}
	Let \(\mathcal{M}=(S,A,T,\mu,R,H)\) be an \momdp with \(m\) objectives, let \(\tau\in\N_{>0}\) be the bidding interval, let \(\beta\in\R_{>0}\) be the maximum bid, let \(\rho\in(0,1)\) be the bid penalty coefficient, and let \(\mathsf{P}\in\{\poorman,\allpay\}\) be the penalty model.
	The induced Markov bidding game between the $m$ local policies is \(\mathcal{G}^{\mathcal{M}}_{\tau,\beta,\rho,\mathsf{P}}=(\hat S,\hat A,\hat T,\hat R,\hat\mu,H)\), where:
	\begin{itemize}
		\item \(\hat S=S\times [1;m]\times [0;\tau-1]\). A state \((s,k,q)\) records the \momdp state \(s\), the current controlling policy \(k\), and the number \(q\) of steps remaining before the next bidding round. When \(q=0\), a new bidding round occurs and the current controller \(k\) is ignored.
		\item \(\hat A=(A\times [0;\beta])^m\) for $\mathsf{P}=\poorman$ and $\hat A=(A\times [0,\beta]^m)$ for $\mathsf{P}=\allpay$. Each policy \(i\) independently chooses an action \(a_i\in A\) and a bid $b_i$, where \(b_i\in[0;\beta]\) or $b_i\in [0,\beta]$ depending on $\mathsf{P}$. We write a joint action as \((\vec a,\vec b)=(a_i,b_i)_{i\in[1;m]}\).
		\item Let \(W(\vec b)=\{i\in[1;m]\mid b_i=\max_j b_j\}\) be the set of highest bidders. The transition kernel \(\hat T:\hat S\times\hat A\to\dist(\hat S)\) is defined by
		\begin{equation}
			\hat T((s,k,q),(\vec a,\vec b),(s',k',q'))=
			\begin{cases}
				T(s,a_k,s') & q>0 \land k'=k \land q'=q-1,\\
				\frac{1}{\abs{W(\vec b)}}T(s,a_{k'},s') & q=0 \land k'\in W(\vec b)\land q'=\tau-1,\\
				0 & \text{otherwise}.
			\end{cases}
		\end{equation}
		Intuitively, the case ``$q>0$'' means it is not yet a bidding round, and the currently active policy $k$ continues to be used, and the case ``$q=0$'' means it is a bidding round, and a new policy is selected via uniform tie-breaking over the set $W$.
		\item For each policy \(i\), the reward is inherited from the \momdp reward \(r_i\), with a bid penalty charged at bidding rounds:
		\begin{equation}
			\hat r_i((s,k,q),(\vec a,\vec b),(s',k',q'))
			=
			r_i(s,a_{k'})
			-
			\rho b_i\cdot \mathbf{1}\!\left[q=0\land (i=k'\lor \mathsf{P}=\allpay)\right].
		\end{equation}
		\item The initial distribution \(\hat\mu\) satisfies \(\hat\mu(s,k,0)=\mu(s)/m\) for every \(k\in[1;m]\), and \(\hat\mu(s,k,q)=0\) for \(q>0\).
	\end{itemize}
\end{definition}
Thus a local policy $i$ in the \momdp $\M$ is a strategy $\pi_i\colon \hat S\to \hat A_i$ in the Markov bidding game $\mathcal{G}^{\mathcal{M}}_{\tau,\beta,\rho,\mathsf{P}}$ maximizing its own expected cumulative values of $\mathbb{E}[\sum_{t=0}^{H-1}\hat r_i^t]$, against every other strategy $\pi_j$ that is maximizing its own expected cumulative reward $\mathbb{E}[\sum_{t=0}^{H-1}\hat r_j^t]$, while the auction rule determines which local action is executed.
We write $\pi_i$ as a pair $(\pi_i^a,\pi_i^b)$, where \(\pi_i^a\colon \hat S\to A\) is the action component and \(\pi_i^b\colon \hat S\to [0;\beta]\) is the bid component  for \poorman and $\pi_i^b\colon \hat S\to [0,\beta]$ is the same for \allpay.
We next analyze the bidding incentives at a single decision state, using continuation values from this game.

\input{theory_new}

%% file: theory_new.tex
%

\smallskip
\noindent\textbf{Equilibrium continuation values.}
We consider local policies \(\pi_i=(\pi_i^a,\pi_i^b)\) that form a Nash equilibrium of the Markov bidding game $\mathcal{G}^{\mathcal{M}}_{\tau,\beta,\rho,\mathsf{P}}$.
For simplicity, assume $H$ is a multiple of $\tau$ (otherwise, extend the horizon to the next multiple of $\tau$).
The policies bid at \textit{bidding rounds} $t\in \{0,\tau,2\tau,\ldots,H-\tau\}$.
Given a policy index $i\in [1;m]$, time $t\leq H$, and state $s\in \hat S$ of $\mathcal{G}^{\mathcal{M}}_{\tau,\beta,\rho,\mathsf{P}}$, we write $V^*_{i,t}$ to denote the value policy $i$ achieves in the rest of the game starting at time $t$ and at state $s$, when all the policies follow a Nash equilibrium.
Since the horizon $H$ is finite, we can define $V^*_{i,t}$ using backward induction over $t$ starting from $t=H$.
Set \(V^*_{i,H}(s)=0\) for every \(i\) and \(s\).
Suppose the game is at state \(s\) and bidding round $t\in \{0,\tau,2\tau,\ldots,H-\tau\}$.

The \textit{winning continuation value} of a policy $i$ is the optimal value it can achieve upon winning bidding:
\begin{equation}
	\widetilde V^{\win}_{i,t}(s)
	\coloneqq
	\max_{\pi_i^a}
	\E^{\pi_i^a}\left[
		\sum_{q=0}^{\tau-1} r_i(s^{t+q},a^{t+q})
		+V^*_{i,t+\tau}(s^{t+\tau})
	\right],
\end{equation}
where the action sequence is generated by policy \(i\) during the time window $[t;t+\tau]$.
Similarly, the \textit{losing continuation value} is the least value it achieves upon losing the bidding:
\begin{equation}\label{eq:losing continuation value}
	\widetilde V^{\lose}_{i,t}(s)
	\coloneqq
	\min_{j\neq i}
	\E^{\pi_j^a}\left[
		\sum_{q=0}^{\tau-1} r_i(s^{t+q},a^{t+q})
		+V^*_{i,t+\tau}(s^{t+\tau})
	\right],
\end{equation}
where the action sequence during $[t;t+\tau]$ is generated by a different policy \(j\) using an action policy $\pi_j^a$ that maximizes $j$'s own winning continuation value.
The \textit{control gain} of policy \(i\) at \((s,t)\) is
\begin{equation}
	G_{i,t}(s)
	\coloneqq
	\widetilde V^{\win}_{i,t}(s)-\widetilde V^{\lose}_{i,t}(s).
\end{equation}
This quantity measures the advantage of winning control for the next bidding interval rather than losing it.
Given bids \(b=(b_1,\ldots,b_m)\), let
$W(b)\coloneqq \{k\in [1;m]\mid b_k=\max_{\ell\in[1;m]}b_\ell\}$ 
be the set of highest bidders, and let \(p_i(b)=\mathbf{1}[i\in W(b)]/\abs{W(b)}\) be the probability that \(i\) wins after uniform tie-breaking.
For a penalty model \(\mathsf{P}\in\{\poorman,\allpay\}\), the \textit{one-round bidding payoff} is
\begin{equation}
	U^{\mathsf{P}}_{i,t}(s,b)
	=
	\begin{cases}
		\widetilde V^{\lose}_{i,t}(s)+p_i(b)\left(G_{i,t}(s)-\rho b_i\right) & \mathsf{P}=\poorman,\\
		\widetilde V^{\lose}_{i,t}(s)+p_i(b)G_{i,t}(s)-\rho b_i & \mathsf{P}=\allpay.
	\end{cases}
\end{equation}
In the \poorman case, the bid cost is paid only in expectation over winning; in the \allpay case, the bid cost is paid regardless of the winner.
Effectively, in the Markov game, the goal of policy $i$ boils down to maximizing the one-round bidding payoff $U^{\mathsf{P}}_{i,t}(s,b)$ at each state $s$ and bidding round $t$. 

A mixed bidding policy for player \(i\) at \((s,t)\) is denoted \(\pi^b_{i,t}(\cdot\mid s)\in\dist([0,\beta])\), and a mixed bidding profile is \(\pi^b_t(\cdot\mid s)=(\pi^b_{1,t}(\cdot\mid s),\ldots,\pi^b_{m,t}(\cdot\mid s))\).
A mixed bidding profile \(\pi^{b*}_t(\cdot\mid s)\) is a Nash equilibrium of the one-round bidding game at \((s,t)\) if, for every player \(i\) and every alternative bid distribution $\pi^b_{i,t}$, where \(\pi^b_{i,t}(\cdot\mid s)\in\dist([0;\beta])\) for $\mathsf{P}=\poorman$ and $\mathsf{P}=\allpay$,
\begin{equation}
	\E_{b\sim\pi^{b*}_t(\cdot\mid s)}\!\left[U^{\mathsf{P}}_{i,t}(s,b)\right]
	\geq
	\E_{b_i\sim\pi^b_{i,t}(\cdot\mid s),\; b_{-i}\sim\pi^{b*}_{-i,t}(\cdot\mid s)}\!\left[U^{\mathsf{P}}_{i,t}(s,b_i,b_{-i})\right].
\end{equation}
For the \poorman setting, the existence of the Nash equilibrium is guaranteed due to the finiteness of the action space $\hat A$~\citep{nash1951noncooperative}, and for the \allpay setting, the same follows from known results in the literature~\citep{baye1996all}.
The equilibrium continuation value is then
\begin{equation}
	V^*_{i,t}(s)
	=
	\E_{b\sim \pi^{b*}_t(\cdot\mid s)}\left[U^{\mathsf{P}}_{i,t}(s,b)\right],
\end{equation}
which completes the backward-induction definition.

\smallskip
\noindent\textbf{Theoretical analysis of the bidding schemes.}
We prove under mild assumptions that if all policies follow the Nash equilibrium then the policy with the highest control gain is guaranteed to obtain control.
Consequently, we obtain a greedy solution to maximize the global reward (the sum of local cumulative rewards).
Further, from Eq.~\eqref{eq:losing continuation value} it follows that the control gains are actually environment-aware, i.e., if the goal of policy $i$ is aligned with other policies (like for the cat feeder, all cats are crowded in one corner), then the control gain of $i$ would be automatically low as it would not mind losing the auction.
We state our two results here but defer the proofs to~\Cref{app:poorman-control-gain-proof,app:allpay-control-gain-proof}.

\begin{theorem}[\poorman favors high control gain]\label{thm:poorman-favors-high-control-gain}
	In the \poorman setting, consider a single bidding round $t$ at the agent state $s\in S$.
  	Suppose the parameters $\rho$ and $\beta$ are so chosen that $\rho\beta > \max_{j\in [1;m]}G_{j,t}(s)$.
	Let $i^*$ be the policy with the largest control gain, such that the following margin condition is fulfilled:
	\begin{align}\label{eq:all pay:gross gain gap}
		G_{i^*,t}(s) - \max_{j\neq i^*}G_{j,t}(s) > \rho.
	\end{align}
	Then in any pure Nash equilibrium, policy $i^*$ wins control.
\end{theorem}

\begin{theorem}[\allpay favors high control gain]\label{thm:allpay-favors-high-control-gain}
	In the \allpay setting with $m\geq 3$, consider a single bidding round $t$ at the agent state $s\in S$.
	Let $i^*$ and $j^*$ be policies for which the following gross gain order holds:
	\begin{align}\label{eqn:allpay urgency gap}
		G_{i^*,t} > G_{j^*,t} > \max_{k\notin \{i^*,j^*\}}G_{k,t}.
	\end{align}
	Then there is no Nash equilibrium with pure strategies, although a unique Nash equilibrium exists with mixed strategies.
	In this Nash equilibrium, policy $i^*$ wins control with probability $1-\sfrac{G_{j^*,t}}{2G_{i^*,t}}$.	
\end{theorem}

While these results do not automatically imply equivalence to the centralized optimal solution, we obtain an environment-aware, greedy approximation.
This provides us a scheduling perspective on the auction: each bidding round picks the objective to receive control based on its urgency.
In Appendix~\ref{app:deterministic-path-planning}, we prove that in deterministic environments with stationary targets and varying deadlines, a Nash equilibrium in our auction mechanism always selects critically urgent goals, if any.
This matches the least slack time first (LSTF) rule, a widely used scheduling heuristic~\citep{brown2020optimal}.

%% file: learning.tex
\section{Learning Local Policies in Markov Bidding Games.}\label{sec:learning-local-policies}

\begin{figure*}[t]
  \centering
  \begin{subfigure}[t]{0.28\textwidth}
    \centering
    \begin{adjustbox}{width=\linewidth,center,trim=1.5cm 0 0 0,clip}
      \input{img/standard_control_overview.tex}
    \end{adjustbox}
    \caption{Monolithic policy}
    \label{fig:standard_control_overview}
  \end{subfigure}
  \hfill
  \begin{subfigure}[t]{0.7\textwidth}
    \centering
    \begin{adjustbox}{width=\linewidth,trim=3cm 0 0 0,clip,center}
      \input{img/modular_auction_overview.tex}
    \end{adjustbox}
    \caption{Auction-based multi-policy framework}
    \label{fig:modular_auction_overview}
  \end{subfigure}

  \caption{Comparison between the standard monolithic setting and our auction-based framework.}
  \label{fig:intro_control_overview}
\end{figure*}

Our goal is to learn the local policy of each player in the Markov bidding game from~\Cref{def:markov-bidding-game} using reinforcement learning.
\Cref{fig:modular_auction_overview} visualizes the architecture of our framework.
Yu et al.~\cite{yu2022mappo} demonstrate that proximal policy optimization (PPO; \cite{schulman2017ppo}) can be effective to train multiple policies concurrently.
We follow their recipe and also train the local policies concurrently with one instance of PPO per policy.
At each iteration, rollouts are collected from the environment by executing the policies, and each PPO instance uses them.
In our experiments, the objectives are symmetric, so we share the actor-critic network across local policies; this reduces the number of learned parameters while preserving separate policy roles through their objective-specific inputs.
This sharing also improves sample efficiency: from the same set of rollouts, the shared network receives PPO updates from all $m$ local-policy objectives rather than from a single objective alone.
Hyperparameters, network sizes, and rollout settings are reported in Appendix~\ref{sec:reprod-details}.

In the practical implementation, we use a discrete bid space rather than the continuous interval in the~\allpay~mechanism.
Each policy selects its bid from a small finite set of levels, so the bid can be represented as an additional categorical action and trained with the same PPO machinery as the environment action.
This discretization is sufficient in our experiments: the bid-level ablation in~\Cref{fig:gridworld-bid-upper-bound} shows that only a small number of levels is needed for effective coordination.

Additionally, to cope with varying number of objectives we employ the \textit{attention pooling} architecture to compress any number of objectives into a fixed-sized embedding~\citep{zaheer2017deep,lee2019set}. 
This design enables the agent to operate with an arbitrary number of policies (i.e., objectives) at deployment time and extract useful information from the global state. 
Formally, the attention pooling module is a permutation-invariant map $\textsf{AP}: \bigcup_{m \geq 1}(\mathbb{R}^{d_z})^m \to \mathbb{R}^{d_h}$ from a variable-size collection of objective vectors to a fixed-dimensional embedding.
Let $(z_1,\ldots,z_m)$ represent the current state of the system, where $z_j$ denotes the information vector for objective $j$; for example, $z_j$ may encode the position of a target.
Each objective information vector is first processed by a fully connected neural network $f_\psi$ to produce an embedding $h_j = f_\psi(z_j)$. 
Given a learned query vector $q$, we compute attention weights and the pooled embedding as 
$
\alpha_j
  =
  \sfrac{\exp(q^\top h_j)}
       {\sum_{\ell=1}^{m} \exp(q^\top h_\ell)},
$ and 
$
h_{\mathrm{pool}}
	  =
	  \sum_{j=1}^{m} \alpha_j h_j 
$.
%
The final embedding produced by the module is $\textsf{AP}(z_1,\ldots,z_m)=h_{\mathrm{pool}}$.
Each local policy has its own attention pooling module. 
The attention-pooling parameters are trained jointly with the rest of the policy network during PPO.
This attention mechanism produces a fixed-dimensional representation that summarizes the varying set of objectives.

%% file: img/standard_control_overview.tex
\newcommand{\stdfigblockwidth}{3.5cm}
\newcommand{\stdfigblockheight}{1.45cm}
\newcommand{\stdfigcornerradius}{2pt}
\newcommand{\stdfigblockinnersep}{6pt}
\newcommand{\stdfigvsep}{2.1cm}
\newcommand{\stdfigtitlesep}{0.9cm}
\newcommand{\stdfigactionright}{1.0cm}
\newcommand{\stdfigstateleft}{1.2cm}
\newcommand{\stdfigactionlabelpos}{0.82}
\newcommand{\stdfigvertlabelpos}{0.5}
\newcommand{\stdfigbboxleftpad}{0.8cm}
\newcommand{\stdfigbboxrightpad}{0.8cm}
\newcommand{\stdfigbboxtoppad}{0.7cm}
\newcommand{\stdfigbboxbottompad}{5cm}
\newcommand{\stdfigflowwidth}{1.05pt}
\newcommand{\stdfigfeedbackwidth}{0.95pt}

\begin{tikzpicture}[
  >=Latex,
  font=\normalsize,
  node distance=\stdfigvsep,
  title/.style={font=\bfseries\large},
  block/.style={
    draw,
    rounded corners=\stdfigcornerradius,
    align=center,
    minimum width=\stdfigblockwidth,
    minimum height=\stdfigblockheight,
    inner sep=\stdfigblockinnersep
  },
  flow/.style={->, line width=\stdfigflowwidth},
  feedback/.style={->, line width=\stdfigfeedbackwidth},
  rewardfeedback/.style={->, line width=\stdfigfeedbackwidth, dotted}
]
  \node[block, fill=blue!10, draw=blue!45!black] (policy) {Policy $\pi$};
  \node[block, below=of policy, fill=gray!12, draw=gray!55!black] (env) {Multi-Objective\\Environment};

  \coordinate (actionright) at ($(policy.east)+(\stdfigactionright,0)$);
  \coordinate (stateleft) at ($(env.west)+(-\stdfigstateleft,0)$);
  \coordinate (actionup) at (actionright |- policy.east);
  \coordinate (actiondown) at (actionright |- env.east);
  \coordinate (stateup) at (stateleft |- policy.west);
  \coordinate (statedown) at (stateleft |- env.west);
  \coordinate (bboxnw) at ($(stateleft |- env.north)+(-\stdfigbboxleftpad,\stdfigbboxtoppad)$);
  \coordinate (bboxse) at ($(actionright |- policy.south)+(\stdfigbboxrightpad,-\stdfigbboxbottompad)$);

  \draw[rewardfeedback] (env) -- node[pos=\stdfigvertlabelpos, right, font=\small] {$\sum_{i=1}^m r_i^{(t)}$} (policy);
  \draw[flow] (policy.east) -- (actionup) -- node[pos=\stdfigvertlabelpos, right, font=\small] {$a^{(t)}$} (actiondown) -- (env.east);
  \draw[feedback] (env.west) -- (statedown) -- node[pos=\stdfigvertlabelpos, left, font=\small] {} node[pos=\stdfigvertlabelpos, right, font=\small] {$\vec s^{(t)}$} (stateup) -- (policy.west);

  \path[use as bounding box]
    (bboxnw)
    rectangle
    (bboxse);
\end{tikzpicture}

%% file: img/modular_auction_overview.tex
\newcommand{\figblockwidth}{4.2cm}
\newcommand{\figblockheight}{1.45cm}
\newcommand{\figpolicywidth}{2.0cm}
\newcommand{\figpolicyheight}{1.0cm}
\newcommand{\figdotgapabove}{0.05cm}
\newcommand{\figdotgapbelow}{0.3cm}
\newcommand{\figapdotgapabove}{0.05cm}
\newcommand{\figapdotgapbelow}{0.3cm}
\newcommand{\figapwidth}{3.25cm}
\newcommand{\figauctionwidth}{3.1cm}
\newcommand{\figauctionheight}{4.3cm}
\newcommand{\figcornerradius}{2pt}
\newcommand{\figblockinnersep}{6pt}
\newcommand{\figsmallinnersep}{4pt}
\newcommand{\fighsep}{2.35cm}
\newcommand{\figauccentersep}{5.65cm}
\newcommand{\figapcentercorrection}{0.575cm}
\newcommand{\figapshift}{2.0cm}
\newcommand{\figvsep}{0.8cm}
\newcommand{\figfeedbackdrop}{1.55cm}
\newcommand{\figtitlesep}{0.9cm}
\newcommand{\fignotesep}{0.8cm}
\newcommand{\figaptitlegap}{0.14cm}
\newcommand{\figenvgap}{0.8cm}
\newcommand{\figmergeleft}{1.5cm}
\newcommand{\figstateleft}{0.4cm}
\newcommand{\figactionlabelpos}{0.24}
\newcommand{\figstatelabelxshift}{-0.9cm}
\newcommand{\figpolicylabelpos}{0.78}
\newcommand{\figpolicylabelxshift}{-1.3cm}
\newcommand{\figpolicylabelyshift}{-0.35cm}
\newcommand{\figaplabelpos}{0.02}
\newcommand{\figaplabelxshift}{2.5cm}
\newcommand{\figaplabelyshift}{0.2cm}
\newcommand{\figcenterpolicylabelxshift}{-0.5cm}
\newcommand{\figrewardlabelxshift}{0.05cm}
\newcommand{\figrewardjoinfrac}{0.3}
\newcommand{\figrewardarcradiusfrac}{0.8}
\newcommand{\figbboxheighttop}{0.5cm}
\newcommand{\figbboxheightbottom}{0.5cm}
\newcommand{\figbboxleftpad}{0cm}
\newcommand{\figbboxrightpad}{1cm}
\newcommand{\figflowwidth}{1.0pt}
\newcommand{\figfeedbackwidth}{0.9pt}

\begin{tikzpicture}[
  >=Latex,
  font=\normalsize,
  node distance=\figvsep and \fighsep,
  title/.style={font=\bfseries\large},
  block/.style={
    draw,
    rounded corners=\figcornerradius,
    align=center,
    minimum width=\figblockwidth,
    minimum height=\figblockheight,
    inner sep=\figblockinnersep
  },
  smallblock/.style={
    draw,
    rounded corners=\figcornerradius,
    align=center,
    minimum width=\figpolicywidth,
    minimum height=\figpolicyheight,
    inner sep=\figsmallinnersep
  },
  apblock/.style={
    draw,
    rounded corners=\figcornerradius,
    align=center,
    minimum width=\figapwidth,
    minimum height=\figpolicyheight,
    inner sep=\figsmallinnersep
  },
  auction/.style={
    draw,
    rounded corners=\figcornerradius,
    align=center,
    minimum width=\figauctionwidth,
    minimum height=\figauctionheight,
    inner sep=\figblockinnersep
  },
  flow/.style={->, line width=\figflowwidth},
  feedback/.style={->, line width=\figfeedbackwidth},
  rewardfeedback/.style={->, line width=\figfeedbackwidth, dotted},
  note/.style={align=center, font=\small}
]
  \node[smallblock, fill=blue!10, draw=blue!45!black] (pol1) {$\pi_1$};
  \node[smallblock, below=of pol1, fill=green!10, draw=green!45!black] (pol2) {$\pi_{i^\star}$};
  \node[smallblock, below=of pol2, fill=orange!12, draw=orange!55!black] (polm) {$\pi_m$};
  \node[font=\Large] at ($(pol2.north)+(0,\figpolicyheight/2+\figdotgapabove)$) {$\vdots$};
  \node[font=\Large] at ($(pol2.south)+(0,-\figdotgapbelow)$) {$\vdots$};

  \node[auction, fill=purple!8, draw=purple!45!black] (auc) at ($(pol2.center)+(\figauccentersep,0)$) {$i^\star \in \arg\max_i b_i^{(t)}$\\[4pt]\footnotesize ties broken\\\footnotesize uniformly at random};
  \node[apblock, fill=green!10, draw=green!45!black] (ap2) at ($(pol2.center)+(-\figauccentersep-\figapcentercorrection-\figapshift,0)$) {$\textsf{AP}_2 : S \to \mathbb{R}^{d_h}$};
  \node[apblock, above=of ap2, fill=blue!10, draw=blue!45!black] (ap1) {$\textsf{AP}_1 : S \to \mathbb{R}^{d_h}$};
  \node[apblock, below=of ap2, fill=orange!12, draw=orange!55!black] (apm) {$\textsf{AP}_m : S \to \mathbb{R}^{d_h}$};
  \node[font=\Large] at ($(ap2.north)+(0,\figpolicyheight/2+\figapdotgapabove)$) {$\vdots$};
  \node[font=\Large] at ($(ap2.south)+(0,-\figapdotgapbelow)$) {$\vdots$};
  \node[note, above=\figaptitlegap of ap1] (aptitle) {\small\bfseries Attention pooling\\[-1pt]\scriptsize\normalfont Extract global state features};
  \node[note, font=\small\bfseries] at (pol2.center |- aptitle.center) {Policies};
  \node[note, font=\small\bfseries] at (auc.center |- aptitle.center) {Auction};
  \coordinate (statebus) at ($(ap2.west)+(-\figstateleft,0)$);
  \coordinate (rewardjoin1) at ($(ap1.east)!\figrewardjoinfrac!(pol1.west)$);
  \coordinate (rewardjoin2) at ($(ap2.east)!\figrewardjoinfrac!(pol2.west)$);
  \coordinate (rewardjoinm) at ($(apm.east)!\figrewardjoinfrac!(polm.west)$);
  \path ($(apm.south)!0.5!(auc.south)$) node[block, below=\figenvgap, fill=gray!12, draw=gray!55!black] (env) {Multi-Objective\\Environment};
  \coordinate (clipnw) at ($(ap1.west |- aptitle.north)+(-\figbboxleftpad,\figbboxheighttop)$);
  \coordinate (clipse) at ($(auc.east |- env.south)+(\figbboxrightpad,-\figbboxheightbottom)$);

  \draw[flow] (pol1.east) -- node[above, font=\small] {$(b_1^{(t)}, a_1^{(t)})$} (auc.west |- pol1.east);
  \draw[flow] (pol2.east) -- node[above, font=\small] {$(b_{i^\star}^{(t)}, a_{i^\star}^{(t)})$} (auc.west |- pol2.east);
  \draw[flow] (polm.east) -- node[above, font=\small] {$(b_m^{(t)}, a_m^{(t)})$} (auc.west |- polm.east);

  \draw[flow] (auc.south) |- node[right, font=\small, pos=\figactionlabelpos] {$a_{i^\star}^{(t)}$} (env.east);

  \draw[line width=\figfeedbackwidth] (env.west) -- (statebus |- env.west) coordinate (statecorner) -- (statebus);
  \draw[feedback] (statebus) |- (ap1.west);
  \draw[feedback] (statebus) -- (ap2.west);
  \draw[feedback] (statebus) |- (apm.west);
  \path ($(env.west)!0.5!(statecorner)$) node[below, font=\small, xshift=\figstatelabelxshift] {$\vec s^{(t)}$};
  \draw[line width=\figflowwidth] (env.west) -- ++(-\figmergeleft,0) coordinate (envrewardcorner);
  \path ($(env.west)!0.5!(envrewardcorner)$) node[below, font=\small, xshift=\figrewardlabelxshift] {$\vec r^{(t)}$};
  \newcommand{\drawrewardarc}[1]{%
    \draw[rewardfeedback]
      let
        \p1=(envrewardcorner),
        \p2=(#1),
        \n1={veclen(\x2-\x1,\y2-\y1)},
        \n2={\figrewardarcradiusfrac*\n1},
        \n3={2*asin(\n1/(2*\n2))},
        \n4={atan2(\y2-\y1,\x2-\x1)}
      in
        (envrewardcorner) arc[start angle=\n4+90+\n3/2, end angle=\n4+90-\n3/2, radius=\n2];
  }
  \drawrewardarc{rewardjoin1}
  \drawrewardarc{rewardjoin2}
  \drawrewardarc{rewardjoinm}
  \draw[feedback] (ap1.east) -- node[pos=\figaplabelpos, above, font=\small, anchor=west, xshift=\figaplabelxshift, yshift=\figaplabelyshift] {$\textsf{AP}_1(s^{(t)})$} node[pos=\figpolicylabelpos, below, font=\small, xshift=\figpolicylabelxshift, yshift=\figpolicylabelyshift] {$r_1^{(t)} - \rho b_1^{(t)}\cdot \mathbf{1}[\text{\allpay}]$} (pol1.west);
  \draw[feedback] (ap2.east) -- node[pos=\figaplabelpos, above, font=\small, anchor=west, xshift=\figaplabelxshift, yshift=\figaplabelyshift] {$\textsf{AP}_{i^\star}(s^{(t)})$} node[midway, below, font=\small, xshift=\figcenterpolicylabelxshift, yshift=\figpolicylabelyshift] {$r_{i^\star}^{(t)} - \rho b_{i^*}^{(t)}  $} (pol2.west);
  \draw[feedback] (apm.east) -- node[pos=\figaplabelpos, above, font=\small, anchor=west, xshift=\figaplabelxshift, yshift=\figaplabelyshift] {$\textsf{AP}_m(s^{(t)})$} node[pos=\figpolicylabelpos, below, font=\small, xshift=\figpolicylabelxshift, yshift=\figpolicylabelyshift] {$r_m^{(t)} - \rho b_m^{(t)}\cdot \mathbf{1}[\text{\allpay}]$} (polm.west);
  \path[use as bounding box]
    (clipnw)
    rectangle
    (clipse);

\end{tikzpicture}

%% file: implementation-eval.tex
\section{Experimental Results}\label{sec:experimental-results}

We evaluate the PPO implementation of our auction framework on evolving reachability objectives and a heterogeneous extension combining feeding priorities and battery charging.
We compare against centralized PPO and DWN, and ablate observations, pooling, and auction hyperparameters.
A detailed description of the environments and the baseline approaches is provided in Appendix~\ref{app:experimental setup}, below we include a summary.
The supplementary material includes our codebase and recorded rollouts of our bidding-based policies on the three base environments.

\textbf{Environments.}
We evaluate our approach on three environments: a mobile cat feeder environment, introduced in Section~\ref{sec:introduction}, and the two Atari games Assault and Air-raid.
In all these environments, the metric being optimized is the cumulative reward across all objectives.

In the Cat Feeder environment, a robot moves on a $30\times 30$ grid to feed up to $m$ cats that move stochastically.
Each cat disappears after a fixed number of steps; feeding it yields a reward, while disappearance incurs a penalty of equal magnitude.
We train with $m=8$ cats in the environment and demonstrate generalization up to $14$.
Each cat is served by a dedicated policy whose observation includes the robot position, the locations of active cats, and their remaining lifetimes.
We ablate global observations with attention pooling, global observations with fixed concatenation (no pooling), and local observations only (blind).

In both the Atari games---Assault (Figure~\ref{fig:assault_preview}) and Air-raid---the agent controls a missile launcher to destroy incoming alien ships while avoiding enemy missiles~\citep{bellemare13arcade}.
Although this is not traditionally seen as a multi-objective environment, we treat each ship as a separate objective and assign a local policy to pursue it, enabling a modular and interpretable design.
In lieu of the standard pixel observations, we employ the object-centric variant of the Atari environments~\citep{delfosse2023ocatari} which provides the environment state information to the policies.
Destroying a target yields reward, losing a life incurs a penalty.
We do not use attention pooling in the Atari games since there are only up to 3 targets.

\textbf{Baseline approaches.}
We compare against two baselines: a single PPO policy optimizing the objective-aggregated reward and deep W-learning (DWN)~\citep{rosero2024multi}.
These baselines also access the same object-centric observations in the Atari games.
The PPO baseline is implemented using CleanRL~\citep{huang2022cleanrl}.
Since we provide distance-based reward shaping to help our policies learn, to ensure a fair comparison, we test nearest-target and earliest-expiry distance shaping, as well as urgency-weighted scalarization of objective rewards, in Cat Feeder.
Hyperparameters are provided in~\Cref{sec:reprod-details}.
PPO parameters for Cat Feeder were tuned using 40 Optuna iterations~\citep{akiba2019optuna} while standard CleanRL Atari hyperparameters were used for the Atari games.
The centralized PPO baseline is trained with the same number of environment samples as our method.
DWN trains one DQN policy per objective together with W-networks that estimate urgency scores of states.
We use the same action window \(\tau\) for DWN.
Training budgets are fixed in advance: 400 PPO iterations for Cat Feeder and Air-Raid, and 150 for Assault.
For each of ten runs, one seed initializes the environment and algorithm; evaluation uses subsequent non-overlapping draws from that seeded stream, rather than independently chosen seed identifiers.

\begin{figure}
    \centering
    \begin{subfigure}{0.32\textwidth}
        \centering
        \includegraphics[width=\textwidth]{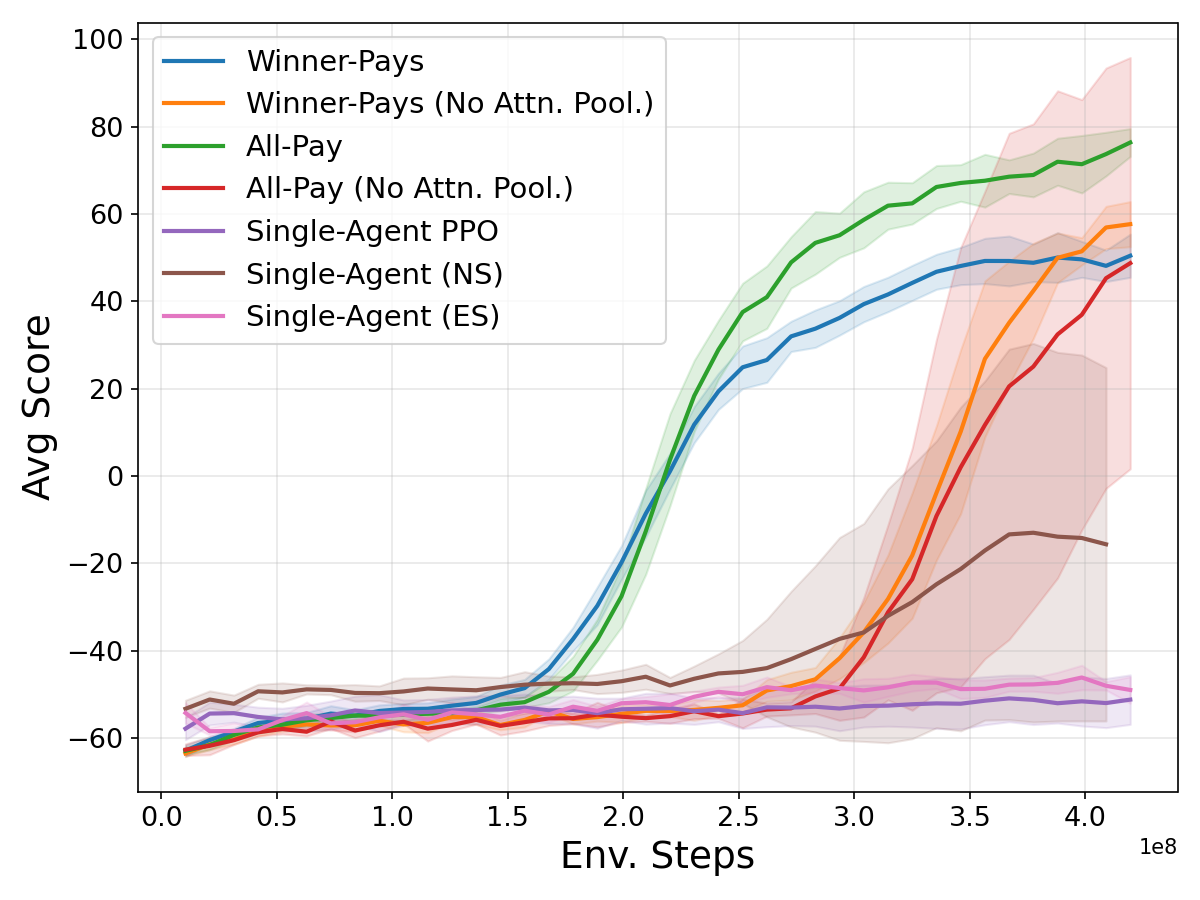}
    \end{subfigure}
    \hfill
    \begin{subfigure}{0.32\textwidth}
        \centering
        \includegraphics[width=\textwidth]{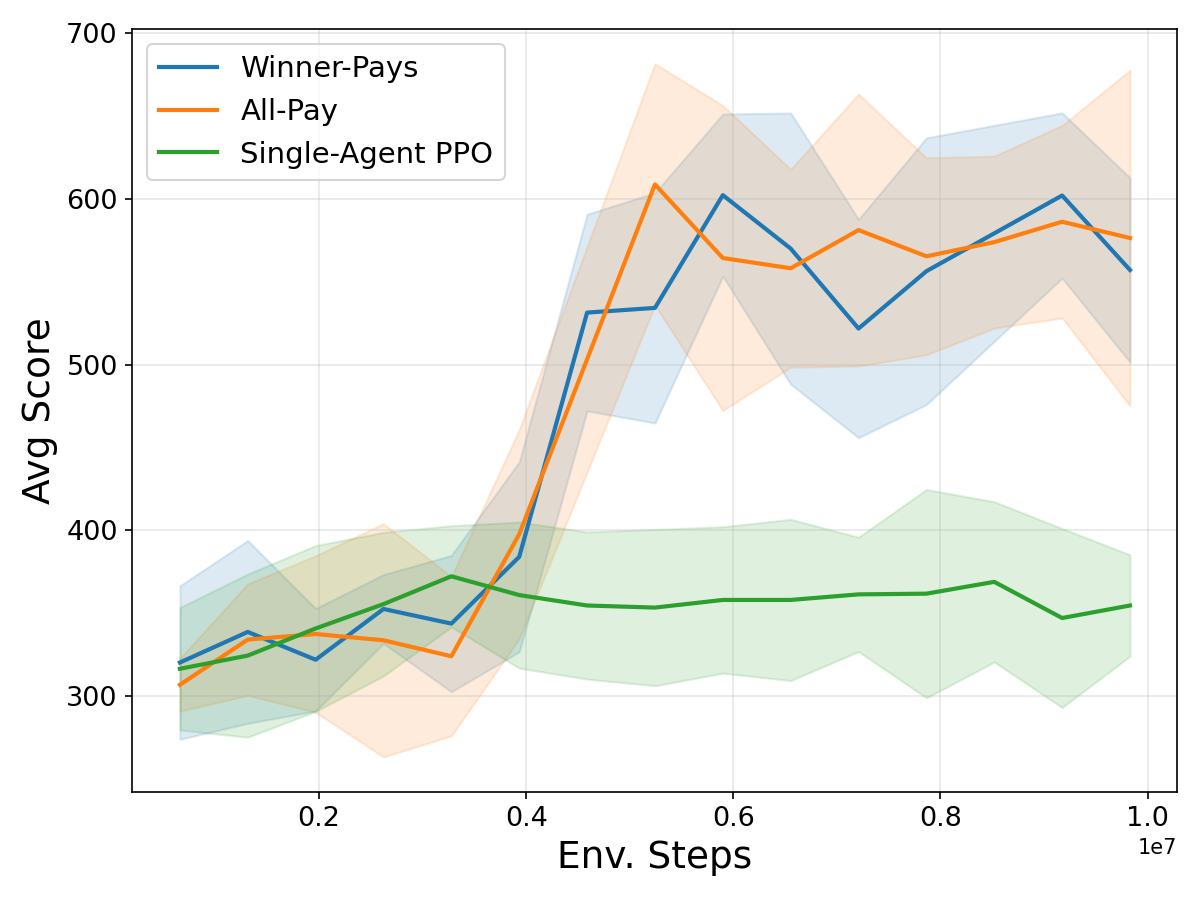}
    \end{subfigure}
    \hfill
    \begin{subfigure}{0.32\textwidth}
        \centering
        \includegraphics[width=\textwidth]{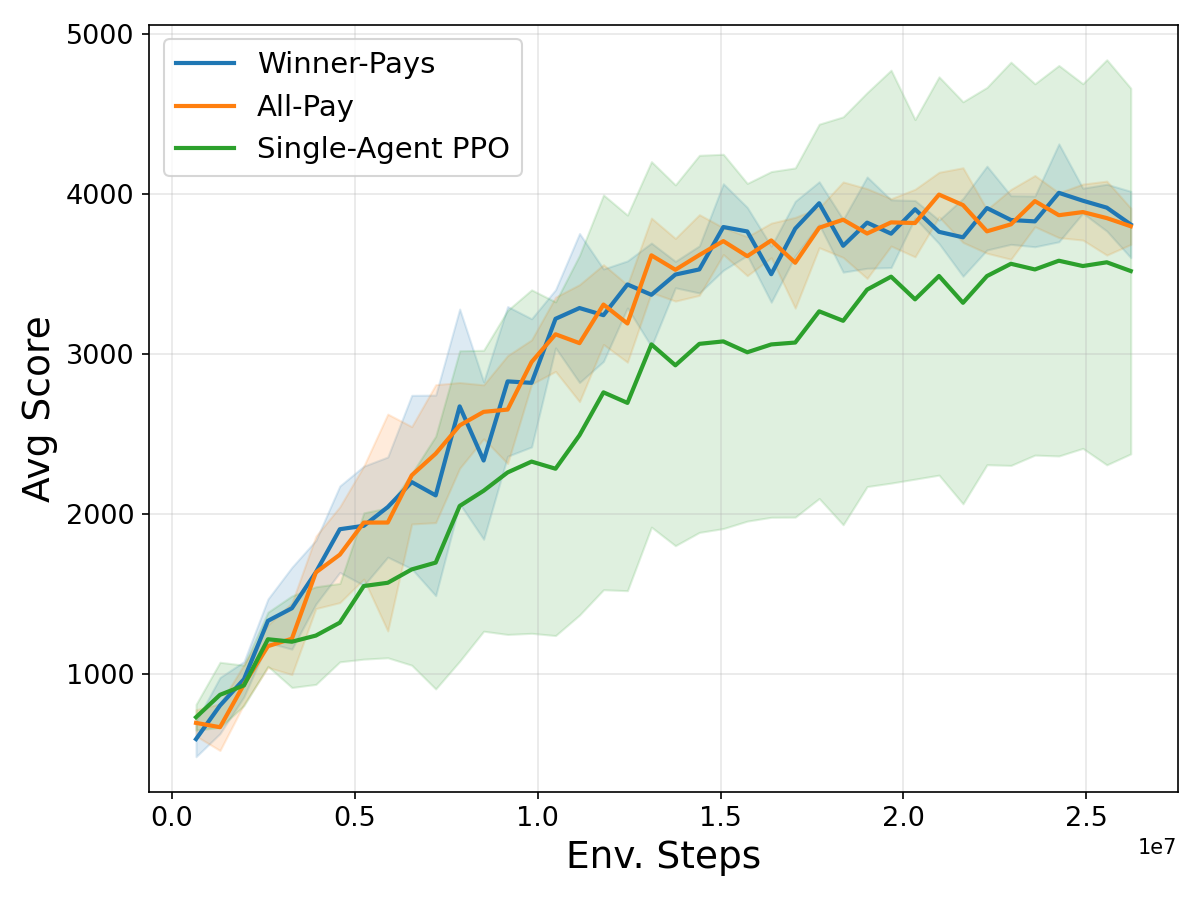}
    \end{subfigure}
    \caption{Learning curves for different methods on Cat Feeder (left), Atari Assault (center), and Atari Air-Raid (right). We plot mean over 10 seeds at each evaluation iteration and the shaded region is \(\pm\) standard deviation across seeds.}
    \label{fig:learning-curves}
\end{figure}

\input{tables/performance_table}

\textbf{Learning dynamics.}\quad \Cref{fig:learning-curves} shows the learning curves for multi-policy bidding and single-policy PPO algorithms on the three environments.
While the curves are more noisy with Atari Assault, they exhibit clearly distinguishable sigmoidal shapes with the other two environments.
The curves for \allpay and \poorman methods start growing around the same time.
\textbf{The variants without attention pooling for targets only start learning much later.}
This validates the use of the attention pooling architecture to learn useful embeddings of targets.
We remark that, based on the multi-agent PPO recipe of Yu et al.~\citep{yu2022mappo}, we also found that using a small clip coefficient (\({\sim}0.05\)) and large batch sizes was crucial to obtain stable but slower convergence of the policies towards an equilibrium.

\textbf{Comparison with baselines.}\quad \Cref{tab:all_methods_comp} lists performance across the three base environments.
Urgency-weighted scalarization raises centralized PPO's Cat Feeder score to $51.78\pm14.31$, comparable in mean to Winner-Pays ($49.44\pm4.19$), while All-Pay achieves $72.52\pm4.23$.
Thus the stronger baseline changes the comparison materially: All-Pay retains an advantage, whereas Winner-Pays does not outperform urgency-weighted PPO in mean score.
Separating objective rewards may ease credit assignment, but these comparisons do not isolate its causal contribution.
In Assault, single extracted frames may omit velocities and projectile directions; the four-frame observations in Air-Raid may reduce this partial observability.
We did not conduct a dedicated frame-stack ablation, so this remains a possible explanation rather than an established effect.

DWN is included as a closely related multi-policy MORL method that learns a state-dependent arbitration signal.
Its low scores suggest that TD-error-based arbitration is less effective in these experiments; they do not establish that its architecture alone causes the difference.
Our empirical comparisons concern this PPO implementation, rather than a general replacement for PPO across domains.

\textbf{Observations and pooling.}\quad
The non-pooled variants retain all target information through fixed concatenation.
Winner-Pays scores with and without pooling ($49.44\pm4.19$ versus $51.72\pm3.97$) are close relative to the across-seed variation; this does not establish superiority of the non-pooled architecture.
Pooling supports variable-size objective sets and earlier learning in the reported curves.
All-Pay with pooling has a higher mean than its blind variant ($67.55\pm3.79$); blind Winner-Pays has a higher mean ($55.80$) but greater variability ($11.42$), so the observations do not show a consistent final-score benefit for both auction mechanisms.

\textbf{Heterogeneous objectives and priorities.}\quad
We extend Cat Feeder to eight cats with feeding priorities from 1 to 4 and three charging stations, all moving stochastically on a $30\times30$ grid.
The robot has a capacity-50 battery; depletion away from a station triggers a tow and refill.
One local policy serves each cat, and a separate charging policy receives a battery-dependent effective bid boost (Appendix~\ref{app:heterogeneous-implementation}).
Feeding rewards scale with cat priority.
Table~\ref{tab:heterogeneous-cat-feeder} reports the sum of priorities of fed cats and battery depletions, against centralized PPO with an unweighted reward sum.
All-Pay achieves a higher feeding score and fewer battery depletions than centralized PPO, demonstrating coordination across feeding and charging objectives; the nonzero depletion rate does not constitute a safety guarantee.
Appendix~\ref{app:heterogeneous-results} reports feeding counts by priority.

\input{tables/heterogeneous_results_table}

\begin{figure}[h]
    \centering
    \begin{subfigure}[t]{0.31\textwidth}
        \centering
        \includegraphics[width=\textwidth]{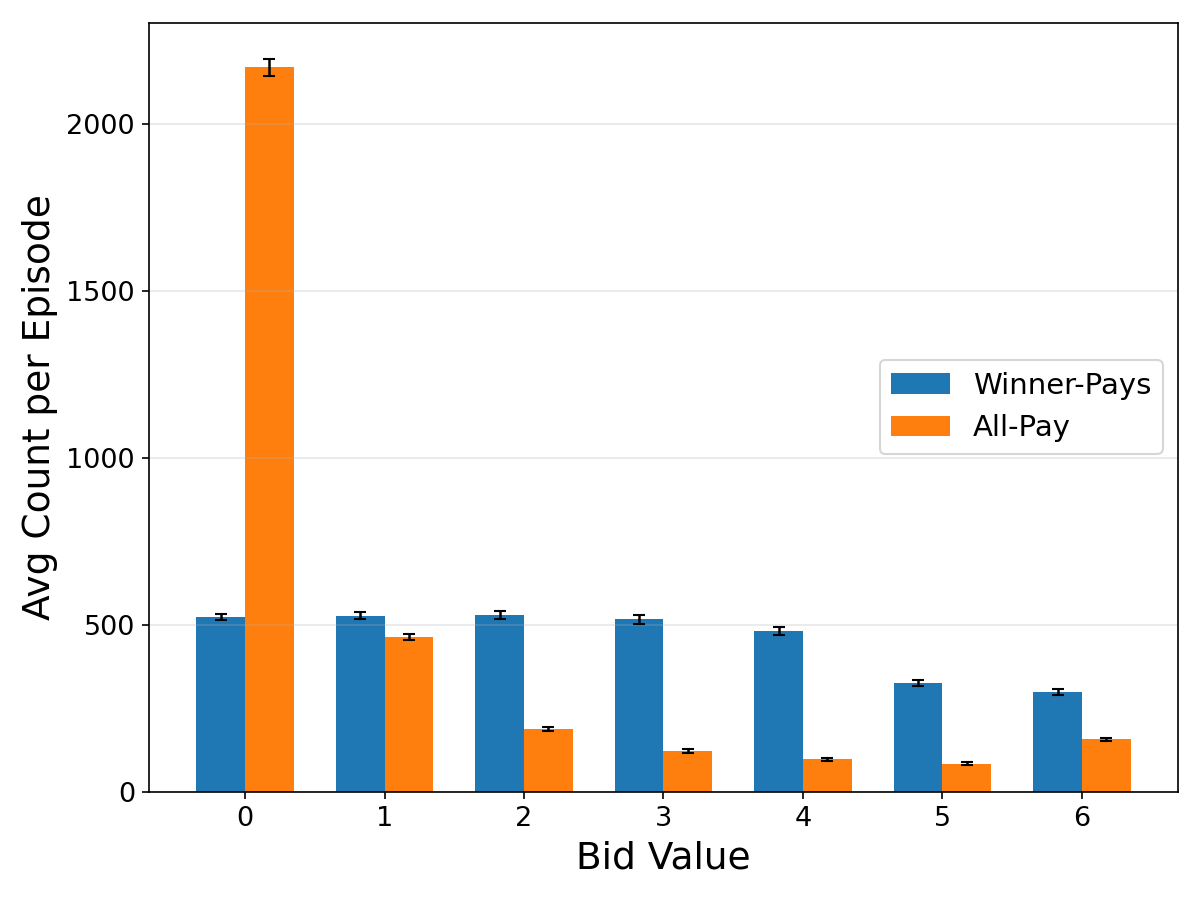}
        \caption{Cat Feeder: Bid distribution}
        \label{fig:dyntar-bid-dist}
    \end{subfigure}
    \hfill
    \begin{subfigure}[t]{0.25\textwidth}
        \centering
        \includegraphics[width=\textwidth]{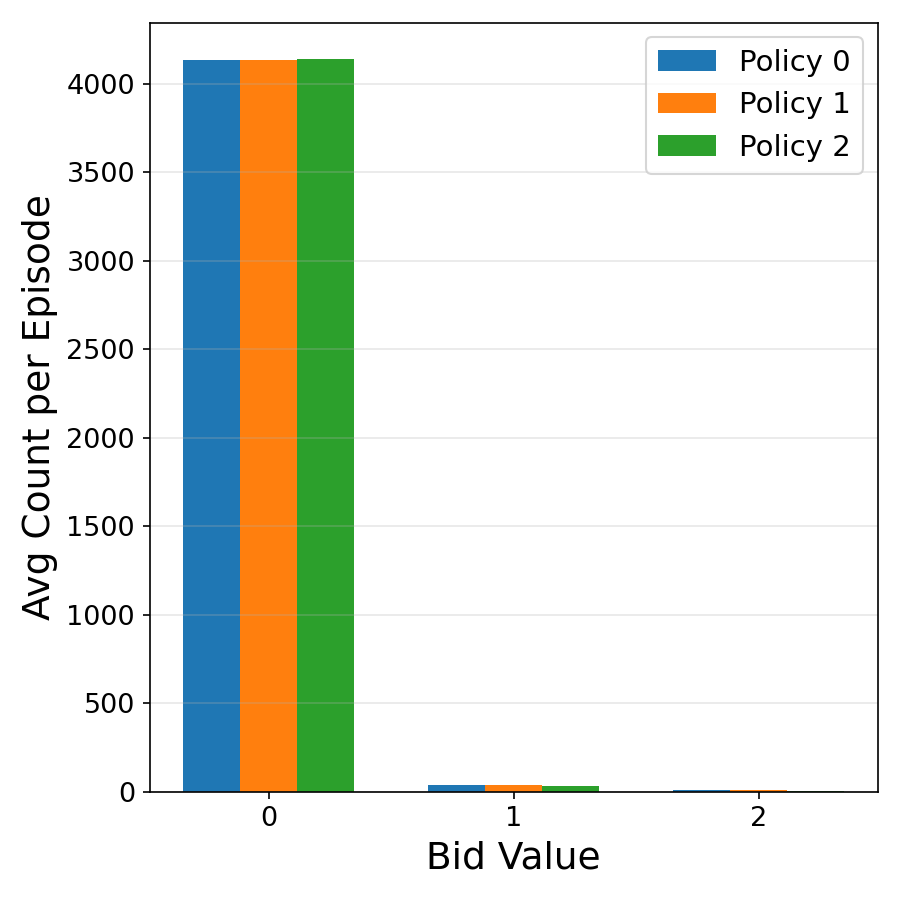}
        \caption{Atari Assault: Bid distribution (\poorman)}
        \label{fig:assault-bid-dist}
    \end{subfigure}
    \hfill
    \begin{subfigure}[t]{0.25\textwidth}
        \centering
        \includegraphics[width=\textwidth]{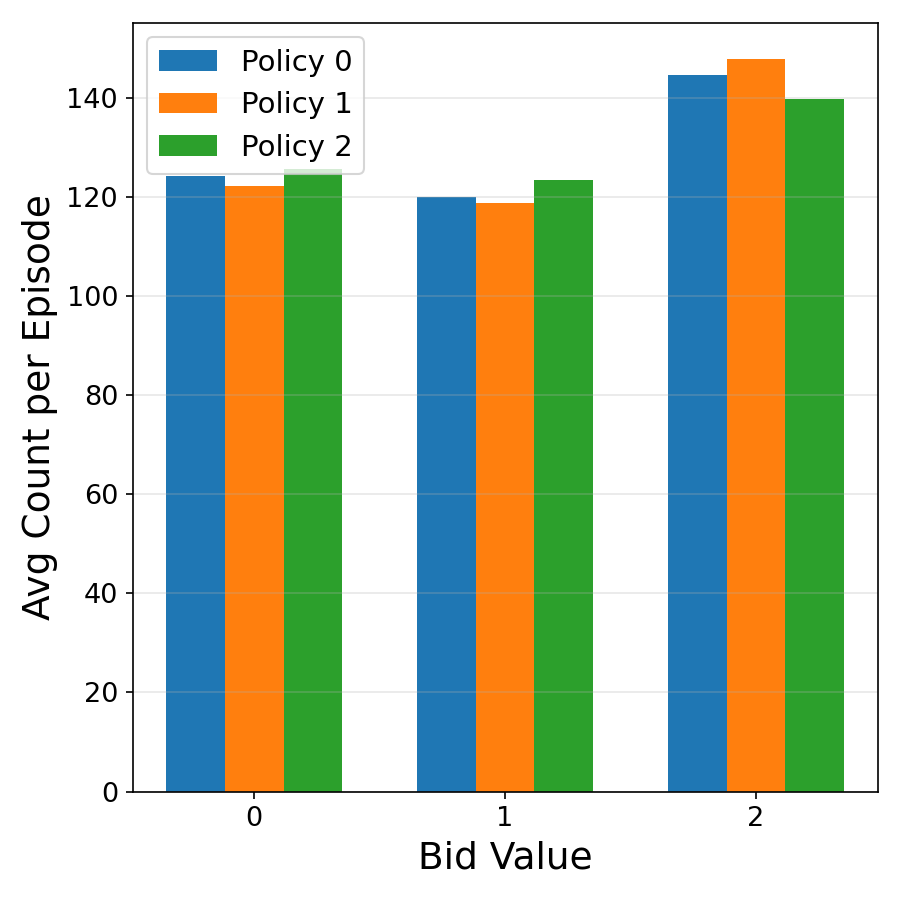}
        \caption{Atari Air-Raid: Bid distribution (\poorman)}
        \label{fig:airraid-bid-dist}
    \end{subfigure}
    \caption{Comparison of bid distributions across environments and methods. Bars show aggregate bid counts for each bid value across all policies.}
    \label{fig:control-bid-dist}
\end{figure}

\begin{figure}
    \centering
    \begin{subfigure}[b]{0.3\textwidth}
        \centering
        \includegraphics[width=\textwidth]{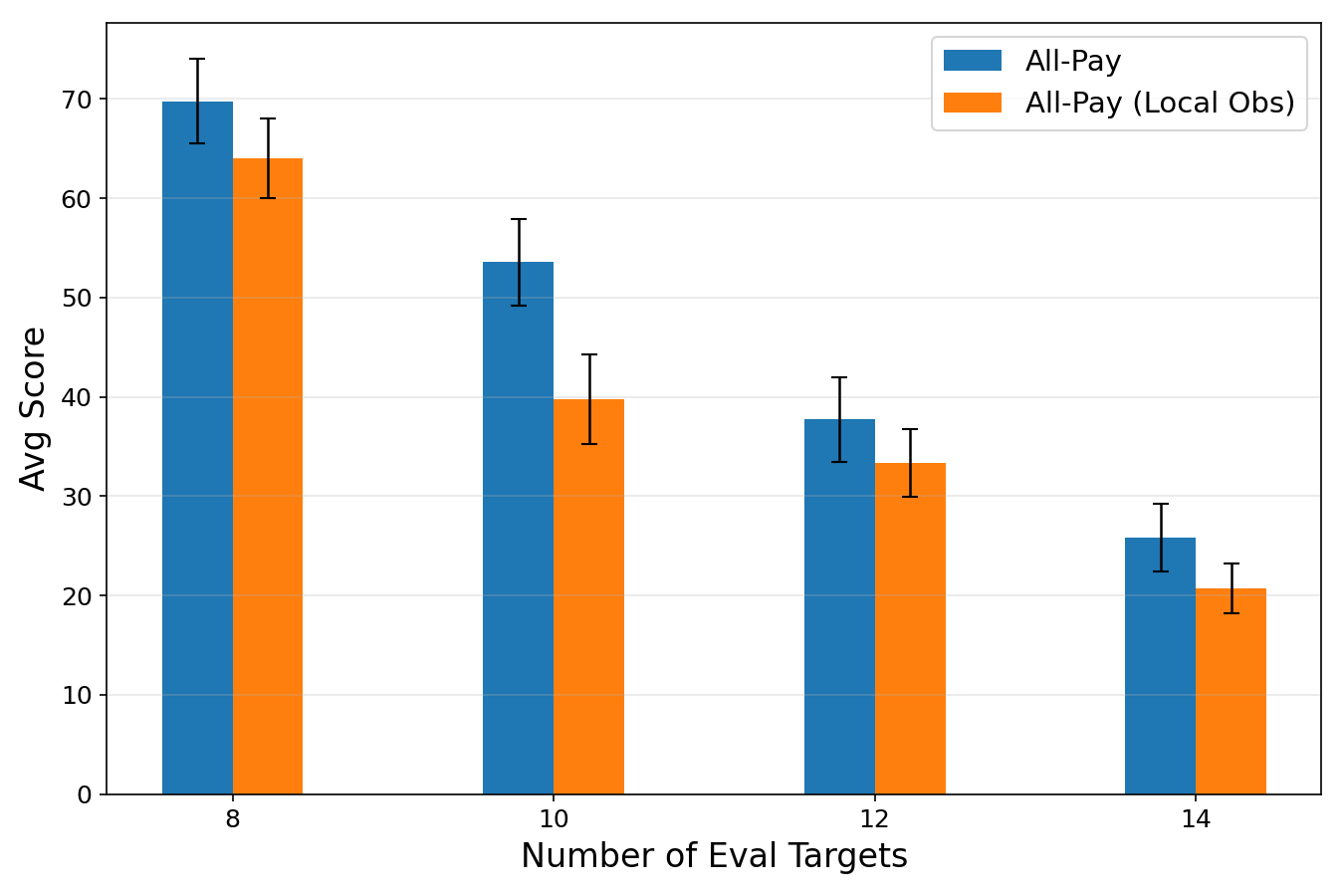}
        \caption{Scaling number of targets.}
        \label{fig:bidding-mechanism-scaling}
    \end{subfigure}
    \hfill
    \begin{subfigure}[b]{0.3\textwidth}
        \centering
        \includegraphics[width=\textwidth]{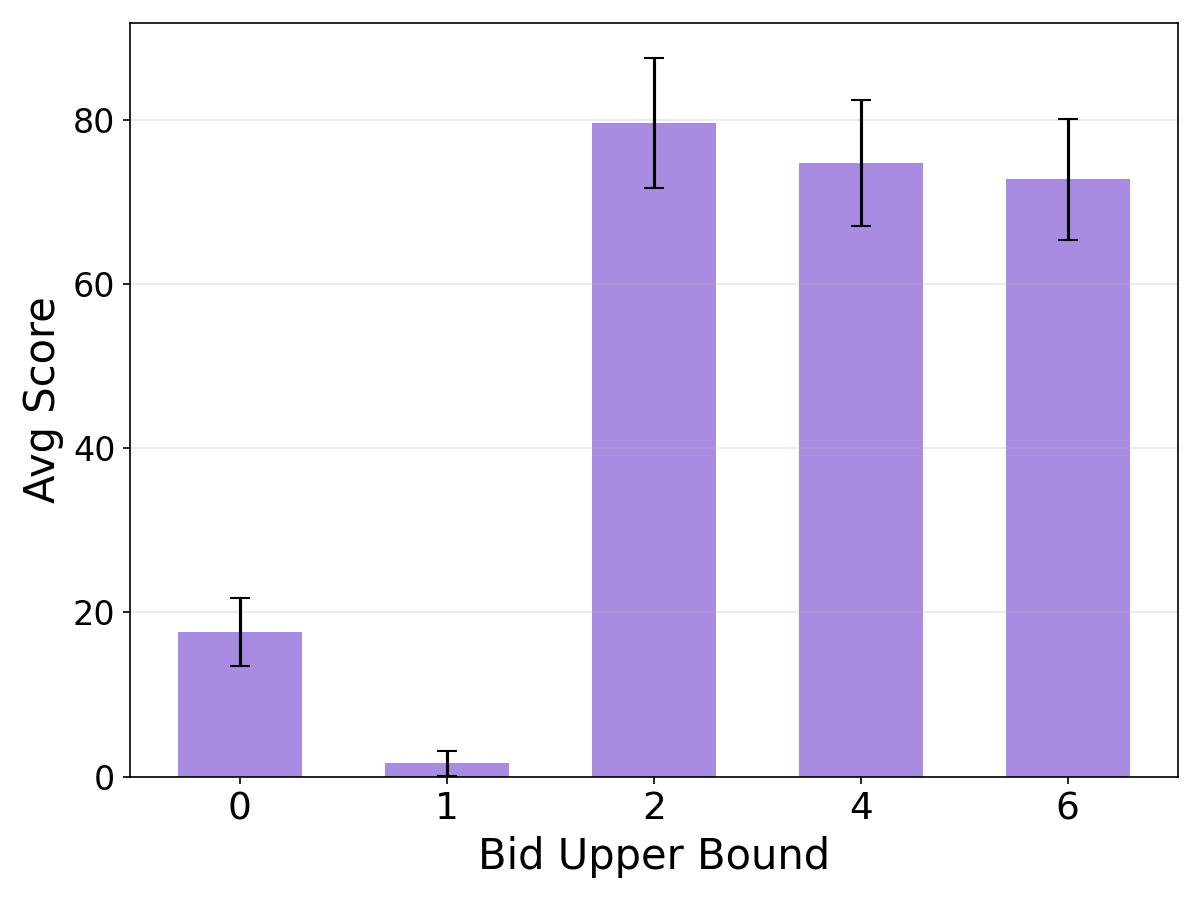}
        \caption{Ablation: Bid upper bound}
        \label{fig:gridworld-bid-upper-bound}
    \end{subfigure}
    \hfill
    \begin{subfigure}[b]{0.3\textwidth}
        \centering
        \includegraphics[width=\textwidth]{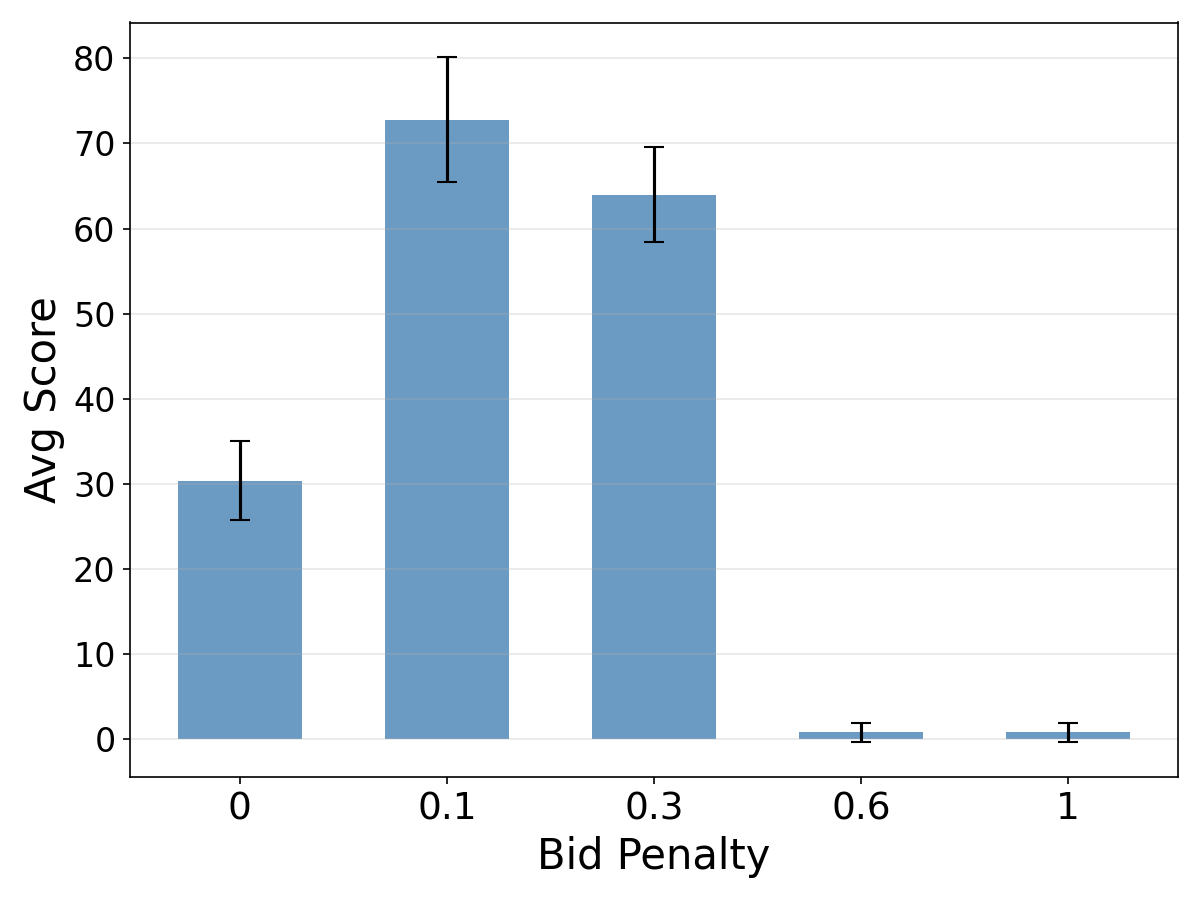}
        \caption{Ablation: Bid penalty}
        \label{fig:gridworld-bid-penalty}
    \end{subfigure}
    \caption{Scaling number of targets and ablations on the Cat Feeder environment.}
    \label{fig:ablations}
\end{figure}

\textbf{Policy behaviors and interpretability.}\quad
Video rollouts illustrate control allocation as each policy pursues its assigned objective; completion need not occur within one action window. \Cref{fig:control-bid-dist} presents the bid distributions.
In Assault, bids are usually 0, making allocation close to random tie-breaking; the gains here may arise primarily from objective decomposition rather than learned urgency arbitration. In contrast, the Cat Feeder environment shows a wider bid spread (0--6). Here, \allpay yields more conservative policies with mostly low bids, while \poorman produces a more even distribution, reflecting greater willingness to bid high.
Finally, in Air-raid, we see policies bidding the highest value more often, although they also bid lower values.
This is because there is a constant stream of enemy ships that each policy has to destroy and so they want to take control more often to avoid missing out on points.

\textbf{Scaling number of targets.}\quad We conduct an experiment with the Cat Feeder environment in which the number of cats seen during testing is larger than that during training.
This is achieved by deploying more copies of the same policy that optimize the additional objectives.
\Cref{fig:bidding-mechanism-scaling} evaluates transfer from eight training targets to as many as fourteen testing targets without retraining.
The plot uses eight times mean per-episode average performance; scaling selects a single best periodic checkpoint rather than averaging the final five.
With fixed episode length, more requests can expire as target count grows (Appendix~\ref{app:scaling-coverage}).
The minimum reported per-target completion count remains positive at every tested size (Appendix~\ref{app:scaling-coverage}), supporting continued service across all targets rather than constant performance as their number grows.

\textbf{Hyperparameter ablations.}\quad \Cref{fig:gridworld-bid-upper-bound} shows that at least two bid levels are needed for good performance; with only one level, policies cannot compete effectively. \Cref{fig:gridworld-bid-penalty} evaluates the bid penalty factor: without penalties all policies overbid, hurting performance, while overly large penalties discourage bidding. Additional action-window ablations are reported in Appendix~\ref{app:additional-experimental-results}.

%% file: tables/performance_table.tex
\begin{table}
  \caption{Scores: mean (standard deviation) across 10 seeds. For all PPO variants, including local observations only, each seed average uses its last 5 available evaluation checkpoints (evaluations every 10 PPO iterations, 20 rollouts each). Cat Feeder score is completed minus expired requests; Atari scores are game scores.}
  \label{tab:all_methods_comp}
  \centering
  \begin{tabular}{llll}
    \toprule
    {\bf Method} & {\bf Cat Feeder} & {\bf Atari Assault} & {\bf Atari Air-raid} \\
    \midrule
    Winner-Pays & $49.44$ {\small\textcolor{gray}{$(4.19)$}} & $556.05$ {\small\textcolor{gray}{$(29.23)$}} & $\textbf{3902.55}$ {\small\textcolor{gray}{$(90.44)$}} \\
    \textcolor{gray}{\itshape\hspace{1em}w. No Attn. Pool.} & \textcolor{gray}{\itshape $51.72$ {\small$(3.97)$}} & \textcolor{gray}{\itshape ---} & \textcolor{gray}{\itshape ---} \\
    \textcolor{gray}{\itshape\hspace{1em}w. Local Obs. Only} & $55.80$ {\small\textcolor{gray}{$(11.42)$}} & --- & --- \\
    All-Pay & $\textbf{72.52}$ {\small\textcolor{gray}{$(4.23)$}} & $\textbf{576.64}$ {\small\textcolor{gray}{$(43.19)$}} & $\textbf{3870.45}$ {\small\textcolor{gray}{$(76.72)$}} \\
    \textcolor{gray}{\itshape\hspace{1em}w. No Attn. Pool.} & \textcolor{gray}{\itshape $37.72$ {\small$(50.32)$}} & \textcolor{gray}{\itshape ---} & \textcolor{gray}{\itshape ---} \\
    \textcolor{gray}{\itshape\hspace{1em}w. Local Obs. Only} & $67.55$ {\small\textcolor{gray}{$(3.79)$}} & --- & --- \\
    \midrule
    DWN & $-10.27$ {\small\textcolor{gray}{$(2.38)$}} & $310.80$ {\small\textcolor{gray}{$(159.54)$}} & $2438.38$ {\small\textcolor{gray}{$(625.97)$}} \\
    Single-Policy PPO & \textcolor{gray}{\itshape ---} & $\textbf{358.60}$ {\small\textcolor{gray}{$(25.99)$}} & $\textbf{3549.33}$ {\small\textcolor{gray}{$(1170.14)$}} \\
    \textcolor{gray}{\itshape\hspace{1em}w. No Reward Shaping} & \textcolor{gray}{\itshape $-51.59$ {\small$(5.33)$}} & \textcolor{gray}{\itshape ---} & \textcolor{gray}{\itshape ---} \\
    \textcolor{gray}{\itshape\hspace{1em}w. Nearest Shaping} & \textcolor{gray}{\itshape $-14.15$ {\small$(41.88)$}} & \textcolor{gray}{\itshape ---} & \textcolor{gray}{\itshape ---} \\
    \textcolor{gray}{\itshape\hspace{1em}w. Expiry Shaping} & \textcolor{gray}{\itshape $-47.62$ {\small$(1.52)$}} & \textcolor{gray}{\itshape ---} & \textcolor{gray}{\itshape ---} \\
    \textcolor{gray}{\itshape\hspace{1em}w. Urgency Weighting} & $51.78$ {\small\textcolor{gray}{$(14.31)$}} & --- & --- \\
    \bottomrule
  \end{tabular}
\end{table}

%% file: tables/heterogeneous_results_table.tex
\begin{table}[t]
  \caption{Heterogeneous Cat Feeder: mean (standard deviation) of priority-weighted feeding score and battery depletions. Higher feeding scores and fewer depletions are better.}
  \label{tab:heterogeneous-cat-feeder}
  \centering
  \begin{tabular}{lrr}
    \toprule
    Method & Feeding priority sum & Battery depletions \\
    \midrule
    All-Pay & $194.56\;(4.67)$ & $0.07\;(0.05)$ \\
    Centralized PPO & $111.86\;(57.08)$ & $3.81\;(5.03)$ \\
    \bottomrule
  \end{tabular}
\end{table}

%% file: conclusion.tex
\section{Concluding Remarks}

We propose a novel modular framework for solving multi-objective RL tasks, where each objective is pursued by a dedicated local policy, and policies bid for the right to execute their action.
The advantages of the framework over the baselines include: better learning and policy performances, and superior adaptability in the face of appearing and disappearing objectives at runtime.

\textbf{Limitations.} 
Our auction-based mechanism selects policies based on the relative urgency of the most urgent policy compared to the rest.
In doing so, sometimes the framework could sacrifice overall optimality: 
in the cat feeder example, if the most urgent target is isolated in one corner, and the rest are concentrated in the opposite corner, the agent could prioritize the single urgent task over the \emph{majority} of less urgent tasks, even if pursuing the majority would lead to greater overall payoff.

\textbf{Future directions.}
Several directions exist.
First, we will attempt to mitigate the above limitation by extending our auction-mechanism to include ``group urgencies'' into considerations, e.g., by incorporating mechanisms like majority voting.
Second, we assumed that all objectives belong to the same family, which we will relax in the future.
In some cases, a lexicographic ordering of priorities over objective types would be interesting, like safety usually gets higher priority than the rest.
Finally, in our experiments, we observe that the auction mechanism treats all objectives \textit{fairly}, whereas monolithic policies tend to focus on one or few objectives.
If true in general, it could imply that our framework achieves much more than merely maximizing the sum of individual payoffs: we will continue investigating this in future work.

%% file: deterministic-appendix.tex
\section{A Special Case of \momdps: Deterministic Time-Bounded Reachability with Multiple Stationary Goals}
\label{app:deterministic-path-planning}

We now  isolate a deterministic path-planning special case that provides useful intuition for the auction mechanism.
This setting is simple enough that we can even show that our auction-based framework would coincide with one of the standard algorithms (called LSTF, introduced later) for solving this problem (see Sec.~\ref{app:lstf-analysis}).
This shows that, while in the most general \momdp setting we are not able to make theoretical claims on optimality, additional assumptions sometime will bridge this gap.

\subsection{The Model and the Shape of the Optimal Policy}

The special case can be stated directly as follows.
An agent moves without uncertainty in a location space \(L\), starting from \(l_0\in L\).
There are \(m\) stationary goals \(g_1,\ldots,g_m\), where goal \(g_i\) has deadline \(d_i\in\mathbb{N}_{>0}\), reward \(r_i\in\mathbb{R}_{>0}\), and missed-deadline penalty \(-M\).
Let \(d(x,g_i)\) denote the minimum travel time from location \(x\) to \(g_i\).
The agent receives \(r_i\) if it reaches \(g_i\) before \(d_i\), and receives \(-M\) otherwise.

Equivalently, the state is a pair \((x,t)\in L\times\mathbb{N}_{\geq 0}\), where \(x\) is the agent location and \(t\) is time.
Actions are movement commands in \(L\), the transition map is deterministic, the initial state is \((l_0,0)\), and the finite horizon \(H\) is chosen large enough to include the relevant deadlines.
The objective is to maximize the sum of the local rewards over all goals.

\begin{proposition}\label{prop:scheduling is the optimal choice}
  In the deterministic stationary-goal setting, there exists a permutation \(\sigma^*\in\Sigma_m\), where \(\Sigma_m\) is the set of all permutations of \([1;m]\), such that an optimal policy pursues the goals in the order
  \[
    \sigma^*(1),\ldots,\sigma^*(m),
  \]
  and follows shortest paths from \(l_0\) to \(g_{\sigma^*(1)}\) and from \(g_{\sigma^*(q)}\) to \(g_{\sigma^*(q+1)}\) for every \(q\in[1;m-1]\).
\end{proposition}

\begin{proof}
  Fix any deterministic policy and consider the order in which it first reaches goals along one of its induced trajectories.
  Since goals are stationary and rewards depend only on first arrival before the corresponding deadlines, any subpath between two consecutive first-reached goals can be replaced by a shortest path between the same endpoints.
  This replacement cannot delay any reached goal and therefore cannot decrease the total reward.
  Repeating this shortcutting argument yields a policy that visits goals according to some permutation and uses shortest paths between consecutive goals.
  Applying the argument to an optimal policy gives the claimed permutation \(\sigma^*\).
\end{proof}

Thus, in this deterministic special case, policy synthesis reduces to choosing an ordering of the goals.
Computing the globally optimal ordering is a centralized scheduling problem, and it can be expensive even in deterministic metric settings.
This motivates lightweight priority rules that can be implemented by local policies.

A standard deterministic priority rule is \emph{least slack time first} (LSTF)~\citep{brown2020optimal}, which schedules the goal with smallest nonnegative slack
\[
  \slack_i(x,t)\coloneqq d_i-t-d(x,g_i).
\]
Slack is well matched to stationary goals and deterministic travel times, but it does not directly apply to the full \momdp setting with stochastic motion, discounted rewards, missed-deadline penalties, and objectives that may appear or disappear.
For the general setting, our theory instead uses the control gain of each local policy: the marginal continuation-value advantage of granting that policy the next control interval rather than deferring control to another policy.
This reward-aware urgency captures both deadline pressure and whether another policy's action would incidentally help the same objective.

\subsection{The Guarantees Provided by the Auction-Based Framework}
\label{app:lstf-analysis}

We now connect the auction mechanism to least slack time first (LSTF) in the deterministic stationary-goal setting from the previous section.
Recall that each goal \(g_i\) is stationary, the agent moves deterministically, and \(d(x,g_i)\) is the minimum travel time from location \(x\) to goal \(g_i\).
At location \(x\) and time \(t\), define the slack of goal \(i\) by
\[
  \slack_i(x,t)\coloneqq d_i-t-d(x,g_i).
\]
Once \(g_i\) is reached, it disappears from the map and we write \(\slack_i(x,t)=\bot\).
The least slack time first rule schedules a goal with minimum nonnegative slack~\citep{brown2020optimal}.
Formally, at location \(x\) and time \(t\), LSTF selects any goal
\[
  i_{\mathrm{LSTF}}\in
  \arg\min_{i:\slack_i(x,t)\geq 0}
  \slack_i(x,t),
\]
with ties broken arbitrarily.
Thus LSTF gives control to the goal whose deadline is closest after accounting for travel time.

We show that, when one goal is critically urgent and all other goals are safely non-urgent, the auction selects the same goal as LSTF.
At each bidding round, every goal \(i\) is represented by a local policy that proposes an action and submits a nonnegative bid.
The highest bidder controls the agent for the next \(\tau\in\mathbb{N}_{>0}\) time steps, with ties broken uniformly at random.
Bids are bounded by \(\beta\), and \(\rho>0\) is the bid penalty coefficient.
For \poorman, bids are integer values in \([0;\beta]\); for \allpay, bids are real values in \([0,\beta]\).
We analyze two payment rules: in \poorman, only the winning policy pays \(\rho b_i\), while in \allpay, every policy pays \(\rho b_i\).

In the remainder of this appendix, \(s\in L\) denotes the current agent location and \(t\) denotes the current time.
Suppose \(\pi_i\) is the local policy for the \(i\)-th goal, mapping each pair \((s,t)\) to a bid \(b_i\) and an action \(a_i\).
We write \(\pi_i^b\colon (s,t)\mapsto b_i\) and \(\pi_i^a\colon (s,t)\mapsto a_i\) for the bidding and action components.
In this deterministic setting, \(\pi_i^a\) proceeds toward its goal \(g_i\) along a shortest path, and \(T^\tau(s,\pi_i^a)\) denotes the state reached by applying \(\pi_i^a\) for \(\tau\) steps from \(s\).

For a bidding round at time \(t\), define \(V_i^{\win}(s,t,b_i)\) and \(V_i^{\lose}(s,t,b_i)\) as the value to policy \(i\) when it bids \(b_i\) and respectively wins or loses the next bidding interval.
Let \(V_i^{\cont}(s,t)\) denote the Nash equilibrium continuation value of the subgame starting from \((s,t)\).
The value functions are defined by backward induction as
\begin{align*}
  V_i^{\win}(s,t,b_i)
  &\coloneqq
  \gamma^\tau V_i^{\cont}(T^\tau(s,\pi_i^a),t+\tau)-\rho b_i,\\
  V_i^{\lose}(s,t,b_i)
  &\coloneqq
  \gamma^\tau V_i^{\cont}(T^\tau(s,\pi_j^a),t+\tau)-\rho b_i\cdot \mathbf{1}[\text{\allpay}],
\end{align*}
where \(j\) is the policy that wins control when \(i\) loses.
The net gain from winning rather than losing is
\begin{align*}
  \Delta_i^{\mathrm{net}}(s,t,b_i)
  &\coloneqq
  V_i^{\win}(s,t,b_i)-V_i^{\lose}(s,t,b_i)\\
  &=
  \Delta_i^{\mathrm{gross}}(s,t)-\rho b_i\cdot \mathbf{1}[\text{\poorman}],
\end{align*}
where the gross gain
\[
  \Delta_i^{\mathrm{gross}}(s,t)
  \coloneqq
  \gamma^\tau\left[
    V_i^{\cont}(T^\tau(s,\pi_i^a),t+\tau)
    -
    V_i^{\cont}(T^\tau(s,\pi_j^a),t+\tau)
  \right]
\]
measures the value advantage to objective \(i\) of winning the next bidding interval.

\begin{lemma}\label{lem:app-lstf-critical-regime-gap}
  In both \poorman and \allpay settings, if policy \(i^*\) is in the critical regime, i.e., \(\slack_{i^*}(s,t)\in[0,\tau]\), and every other policy \(j\neq i^*\) is in the safe regime, i.e., \(\slack_j(s,t)>\tau\), then
  \begin{align}\label{eq:app-lstf-gross-gain-urgency}
    \Delta_{i^*}^{\mathrm{gross}}(s,t)
    -
    \max_{j\neq i^*}\Delta_j^{\mathrm{gross}}(s,t)
    > M-r_{\max}>0.
  \end{align}
\end{lemma}

\begin{proof}
  Since \(\slack_{i^*}(s,t)\leq \tau\), losing the bidding interval causes \(i^*\) to miss its deadline unless the winning policy reaches \(g_{i^*}\) incidentally during that interval.
  In the worst case represented by the losing continuation value, this incidental help does not occur.
  Hence \(V_{i^*}^{\lose}(s,t,b_{i^*})=-M-\rho b_{i^*}\cdot\mathbf{1}[\text{\allpay}]\).
  On the other hand, if \(i^*\) wins the bidding, then \(V_{i^*}^{\win}(s,t,b_{i^*})\geq -\rho b_{i^*}\), with strict inequality if the goal is reached during the next \(\tau\) steps.
  Therefore
  \[
    \Delta_{i^*}^{\mathrm{gross}}(s,t)\geq M.
  \]

  For any \(j\neq i^*\), the safe-regime condition \(\slack_j(s,t)>\tau\) implies that losing the next bidding interval does not by itself force \(j\) to miss its deadline.
  Thus \(V_j^{\win}(s,t,b_j)\leq \gamma^{\tau+1}r_j-\rho b_j\) and \(V_j^{\lose}(s,t,b_j)\geq-\rho b_j\cdot\mathbf{1}[\text{\allpay}]\), so
  \[
    \Delta_j^{\mathrm{gross}}(s,t)
    \leq
    \gamma^{\tau+1}r_j
    < r_j
    \leq r_{\max}.
  \]
  Combining the two bounds gives
  \[
    \Delta_{i^*}^{\mathrm{gross}}(s,t)-\max_{j\neq i^*}\Delta_j^{\mathrm{gross}}(s,t)
    > M-r_{\max}.
  \]
  Since \(M>r_{\max}\), the right-hand side is positive.
\end{proof}

Lemma~\ref{lem:app-lstf-critical-regime-gap} is the only special-case ingredient needed to relate the auction mechanism to LSTF.
It says that a critically urgent goal has a strictly larger gross gain than every safe goal.
Therefore, in the \poorman setting, if \(0<\rho<M-r_{\max}\), the margin condition in Theorem~\ref{thm:poorman-favors-high-control-gain} is satisfied and the critically urgent LSTF goal wins in any pure Nash equilibrium.
In the \allpay setting, the same gross-gain ordering is the input to Theorem~\ref{thm:allpay-favors-high-control-gain}; hence the mixed equilibrium concentrates control on the critically urgent LSTF goal, with concentration increasing as the urgent goal's gross gain dominates the runner-up's gross gain.

%% file: poorman-proof-appendix.tex
\section{Proof of the \poorman Control-Gain Result}
\label{app:poorman-control-gain-proof}

\begin{proof}[Proof of Theorem~\ref{thm:poorman-favors-high-control-gain}]
	Suppose for contradiction, some other policy $j\neq i^*$ wins control by bidding $b_j$, while policy $i^*$ bids $b_{i^*}<b_j$. 
	As this is a Nash equilibrium, it must hold that 
	$U_{j,t}^{\mathsf{P}}(s,b_j) > \widetilde V^{\lose}_{j,t}(s)$, i.e., $G_{j,t}(s) - \rho b_j > 0$,
	because otherwise, policy $j$ would be better off by losing the bidding instead.
	This gives us the upper bound $\rho b_j < G_{j,t}(s)$.
	It follows from $\rho\beta > \max_{k\in [1;m]}G_{k,t}(s)$ that $\rho b_j < \rho\beta$, and hence $b_j<\beta$, i.e., the policy $j$ does not make the maximum bid.
	
	Now we assumed that policy $i^*$ had bid $b_{i^*}<b_j$, but we show that policy $i^*$ would achieve a higher one-round bidding payoff by bidding $b_j+1$ (valid, since $b_j<\beta\implies b_j+1\leq \beta$), contradicting the Nash equilibrium:	 
$	U_{i^*,t}^{\mathsf{P}}(s,b_j+1) 
	= \widetilde V^{\lose}_{i^*,t}(s)+G_{i^*,t}(s) - \rho(b_j+1) 
	= \widetilde V^{\lose}_{i^*,t}(s)+ G_{i^*,t}(s) - \rho b_j -\rho  
	> \widetilde V^{\lose}_{i^*,t}(s)+ G_{i^*,t}(s) - G_{j,t}(s) -\rho
	> \widetilde V^{\lose}_{i^*,t}(s) = U_{i^*,t}^{\mathsf{P}}(s,b_j)
$.
\end{proof}

%% file: allpay-proof-appendix.tex
\section{Proof of the \allpay Control-Gain Result}
\label{app:allpay-control-gain-proof}

\begin{proof}[Proof of Theorem~\ref{thm:allpay-favors-high-control-gain}]
	It follows from the work of Baye et al.~\cite{baye1996all} that in the Nash equilibrium, the optimal bid of the policies are as follows.
	The policy $i^*$ picks a uniformly random bid $b_i^*$ from $[0,G_{j^*,t}/\rho]$, the policy $j^*$ picks a uniformly random bid $b_j^*$ from $(0,G_{j^*,t}/\rho]$ while picking the bid $0$ with probability $(G_{i^*,t} - G_{j^*,t})/G_{i^*,t}$, and every other policy $k\notin\{i^*,j^*\}$ deterministically picks the bid $0$.
	Clearly, if $b_i^*>b_j^*$, then $i^*$ wins control, and we need to show that $P[b_i^* > b_j^*]  = 1- \sfrac{G_{j^*,t}}{2G_{i^*,t}}$.
	To show this, consider the probability density function $f_{i^*}(b) = \sfrac{\rho }{G_{j^*,t}}$ of $b_{i^*}$, and the cumulative distribution function $F_{j^*}(b) = (G_{i^*,t} - G_{j^*,t})/G_{i^*,t} + \sfrac{\rho b}{G_{i^*,t}}$ of $b_{j^*}$.
	Then,
	\begin{align*}
		P[b_i^* > b_j^*] 
		&= \int_{-\infty}^\infty F_{j^*}(b)f_{i^*}(b)db\\
		&= \int_{0}^{G_{j^*,t}/\rho} \left(\frac{G_{i^*,t} - G_{j^*,t}}{G_{i^*,t}} + \frac{\rho b}{G_{i^*,t}}\right)\cdot \frac{\rho }{G_{j^*,t}} \cdot db \\
		&= \int_{0}^{G_{j^*,t}/\rho} \frac{G_{i^*,t} - G_{j^*,t}}{G_{i^*,t}} \cdot \frac{\rho }{G_{j^*,t}} \cdot db
		+ \int_{0}^{G_{j^*,t}/\rho} \frac{\rho b}{G_{i^*,t}}\frac{\rho }{G_{j^*,t}} \cdot db\\
		&=  \frac{G_{i^*,t} - G_{j^*,t}}{G_{i^*,t}} \cdot \frac{\rho }{G_{j^*,t}} \cdot \frac{G_{j^*,t}}{\rho}
		+ \frac{\rho^2}{G_{i^*,t}G_{j^*,t}}\cdot \frac{\left(G_{j^*,t}\right)^2}{2\rho^2}\\
		&= \frac{G_{i^*,t} - G_{j^*,t}}{G_{i^*,t}} + \frac{G_{j^*,t}}{2G_{i^*,t}}\\
		&= 1-\frac{G_{j^*,t}}{G_{i^*,t}} + \frac{G_{j^*,t}}{2G_{i^*,t}}\\
		&= 1-\frac{G_{j^*,t}}{2G_{i^*,t}}.
	\end{align*}
\end{proof}

%% file: envs-appendix.tex
\section{Experimental Setup (Detailed)}
\label{app:experimental setup}

We consider three multi-objective environments for our empirical studies, namely a mobile cat feeder and the Atari games Assault and Air-Raid.
These environments are described below.

\textbf{Environment I: Cat Feeder.} 
The environment is as described in Section~\ref{sec:introduction} and is modeled using a $30\times 30$ grid with up to $m$ cats that are moving continuously according to unknown stochastic processes.
The value of $m$ was set as $8$ during the training time, whereas it was varied up to $14$ during the testing time.
Each cat disappears after a given number of time steps.
If the agent reaches a cat before it disappears, it receives a reward, and otherwise it receives a penalty of equal magnitude.
To avoid sparsity of the reward function, we provide each policy a small reward for going closer to its target.
Performance is measured as the difference between the number of feeding requests completed and the number of those that expired.
The Cat Feeder environment parameters are listed in~\Cref{tab:cat-feeder-env-params}, and the corresponding multi-policy, single-policy PPO, and DWN training parameters are listed in~\Cref{tab:cat-feeder-mappo-hparams,tab:cat-feeder-ppo-hparams,tab:cat-feeder-dwn-hparams}.

In our modular framework, each cat will be chased by a dedicated local policy.
The observation space of each local policy includes the current position of the robot, the current locations of the active requests, and the time steps until the expiration of the requests.
For local policy $i$, the Cat Feeder reward is
\[
    r_i =
    r_i^{\mathrm{bid}}
    + r_i^{\mathrm{window}}
    + r_i^{\mathrm{dist}}
    + r_i^{\mathrm{reach}}
    + r_i^{\mathrm{expire}} .
\]
Here $r_i^{\mathrm{bid}}$ is the auction payment term, $r_i^{\mathrm{window}}$ is the action-window penalty, $r_i^{\mathrm{dist}}$ rewards distance improvement toward policy $i$'s target, $r_i^{\mathrm{reach}}$ is paid when that target is reached, and $r_i^{\mathrm{expire}}$ penalizes target expiration.
The multi-policy decomposition preserves a separate reward signal for each target.
Centralized PPO can also scalarize per-target rewards, but must infer control allocation and movement from the resulting scalar signal.
The single-policy PPO baseline uses the distance-improvement, target-reached, and target-expiration terms, without bid or action-window penalties.

\textbf{Environment II: Atari Assault.} We consider the Atari game called Assault, where the agent controls a missile launcher to shoot oncoming alien ships while avoiding enemy missiles. The standard approach is to view this as a single-objective environment, where the agent's policy observes the entire image or the 128 byte raw RAM state observations to determine the optimal actions \citep{bellemare13arcade}. 

\begin{figure}[t]
    \centering
    \includegraphics[width=0.3\linewidth]{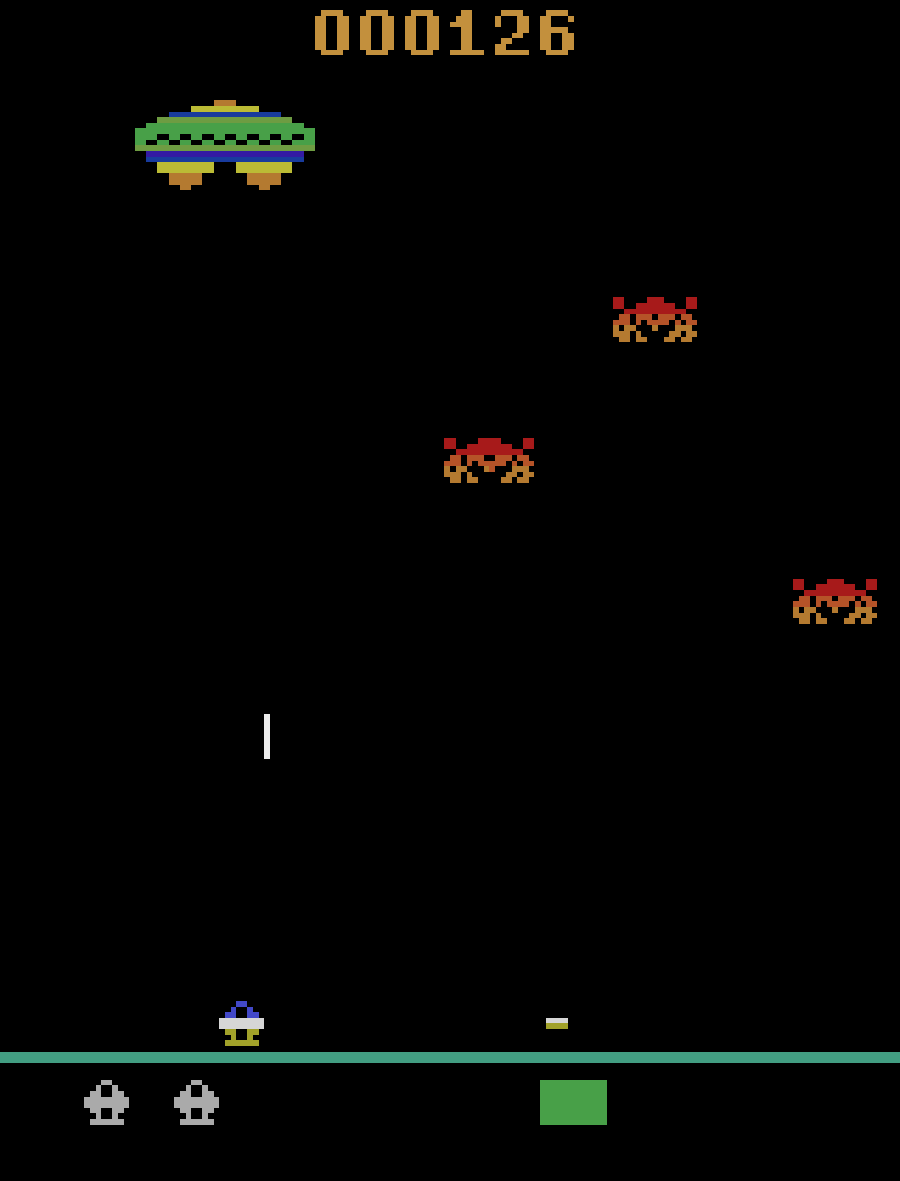}
    \caption{Atari Assault environment.}
    \label{fig:assault_preview}
\end{figure}

For our modular framework, we break the objective down to a multi-objective problem, where each alien ship is viewed as a separate objective, and each local policy attempts to destroy its assigned target while avoiding enemy missiles. 
Not only does this facilitate an improvement in the performance, as we will demonstrate, but also we obtain a transparent design, because which target is being chased is revealed at each point. 
This modular design is made possible by the object-centric variant of the Atari environment \citep{delfosse2023ocatari}, which exposes the states of each individual target and the missiles by decoding the RAM states. 
The full observation space of each local policy also include the current position of the missile launcher, the current health, and the number of remaining lives. 
The local policy gets a reward or a penalty upon destroying its target enemy ship or losing a life, respectively. 
For local policy $i$, the Assault reward is
\[
    r_i =
    r_i^{\mathrm{destroy}}
    + r_i^{\mathrm{life}}
    + r_i^{\mathrm{overheat}}
    + r_i^{\mathrm{hot}}
    + r_i^{\mathrm{bid}}
    + r_i^{\mathrm{window}} .
\]
The destroy reward is credited to the policy assigned to the destroyed enemy, while life-loss, overheat, and fire-while-hot penalties discourage unsafe firing behavior.
The bid and action-window terms account for the auction mechanism.
Lastly, we use the game score as the metric for evaluation.
The Assault environment parameters are listed in~\Cref{tab:assault-env-params}, and the corresponding multi-policy, single-policy PPO, and DWN training parameters are listed in~\Cref{tab:assault-mappo-hparams,tab:assault-ppo-hparams,tab:assault-dwn-hparams}.

\textbf{Environment III: Atari Air-Raid.} Air-Raid uses the same modular Atari setup: each incoming enemy is treated as an objective and assigned to a local policy.
Positive score changes are interpreted as enemy kills, and each kill yields a destroy reward proportional to the score increase.
We also penalize near misses from enemy missiles: if an enemy missile enters a fixed box around the player without causing a life loss, the controlling policy receives a near-hit penalty.
For local policy $i$, the Air-Raid reward is
\[
    r_i =
    r_i^{\mathrm{destroy}}
    + r_i^{\mathrm{near\text{-}hit}}
    + r_i^{\mathrm{bid}} .
\]
Destroy rewards are credited to the policy associated with the hit enemy when that policy owns the player missile that caused the score event, near-hit penalties are assigned to the policy currently controlling the player, and bid penalties are applied according to the auction rule.
The single-policy PPO baseline uses the destroy and near-hit terms without bid penalties.
Additional training and implementation details are reported in~\Cref{sec:reprod-details}.

\textbf{Baseline approaches.}
We compare our method against two baselines: a single policy PPO with weighted rewards and deep W-learning (DWN; \citep{rosero2024multi}). 
For the original PPO baselines in the homogeneous environments, we use the total reward across objectives at each step.
We additionally evaluate urgency-weighted scalarization in Cat Feeder, where objective reward components are weighted dynamically by urgency.
The precise weighting and update order are given in Appendix~\ref{app:urgency-scalarization}.
We note that in Atari Assault and Air-raid, the single-policy PPO reward includes enemy-destroy rewards, raw score rewards, life-loss penalties, overheat penalties, and fire-while-hot penalties, but no bid or action-window terms.
The observation space remains the same as the multi-policy setting. 
In the Cat Feeder environment, we implement three reward mechanisms to help guide learning: one only with the sparse rewards from the environment, one with distance-based dense reward with respect to the nearest feeding request (denoted Nearest Shaping/NS in the experiments), and finally one where we provide distance-based reward with respect to the request that has the least amount of time until expiration (denoted Expiry Shaping/ES).
The single-policy PPO baseline does not use an action window: because there is no policy-switching decision, it selects a new action at every environment step. The action window \(\tau\) applies only to multi-policy methods, where it determines how long the selected local policy retains control after a bidding or W-network selection event.
The PPO implementation is taken from CleanRL~\citep{huang2022cleanrl}.
We use the specified PPO hyperparameters from CleanRL for the Atari games and we tune the hyperparameters for the Cat Feeder environment with 40 Optuna iterations~\citep{akiba2019optuna}.

DWN is another multi-policy method for multi-objective RL that uses DQN for each of the local policies and additionally learns a W-network for each policy that predict scalar weights. 
These weights are used to choose which policy's action should be played at each step. 
For temporal consistency among multi-policy methods, DWN uses the same action window \(\tau\) as our bidding methods.
The W-network is trained alongside the Q-network by updating it with a gradient step to predict the negative temporal difference (TD) error when the policy does not get to play its action.
The assertion here is that the negative TD-error quantifies how important it is for the policy to play its action. 
The implementation is borrowed from the repository of \citep{rosero2024multi}. All hyperparameters for the baselines are listed in~\Cref{sec:reprod-details}.

\textbf{Heterogeneous Cat Feeder.}
The extension adds priority-scaled feeding rewards and a distinct charging objective to the same gridworld.
Eight cat policies compete with one charging policy; the three charging stations and the cats move stochastically.
The charging policy has a battery-dependent effective bid boost to encode its known priority.
Battery dynamics, effective bids, network sharing, and the centralized reward are specified in Appendix~\ref{app:heterogeneous-implementation}.
The feeding metric sums priorities of completed requests (priorities 1--4), and battery depletions are reported separately.

%% file: additional-experimental-results.tex
\section{Additional Experimental Results}
\label{app:additional-experimental-results}

\begin{figure}[h]
    \centering
    \begin{subfigure}[t]{0.31\textwidth}
        \centering
        \includegraphics[width=\textwidth]{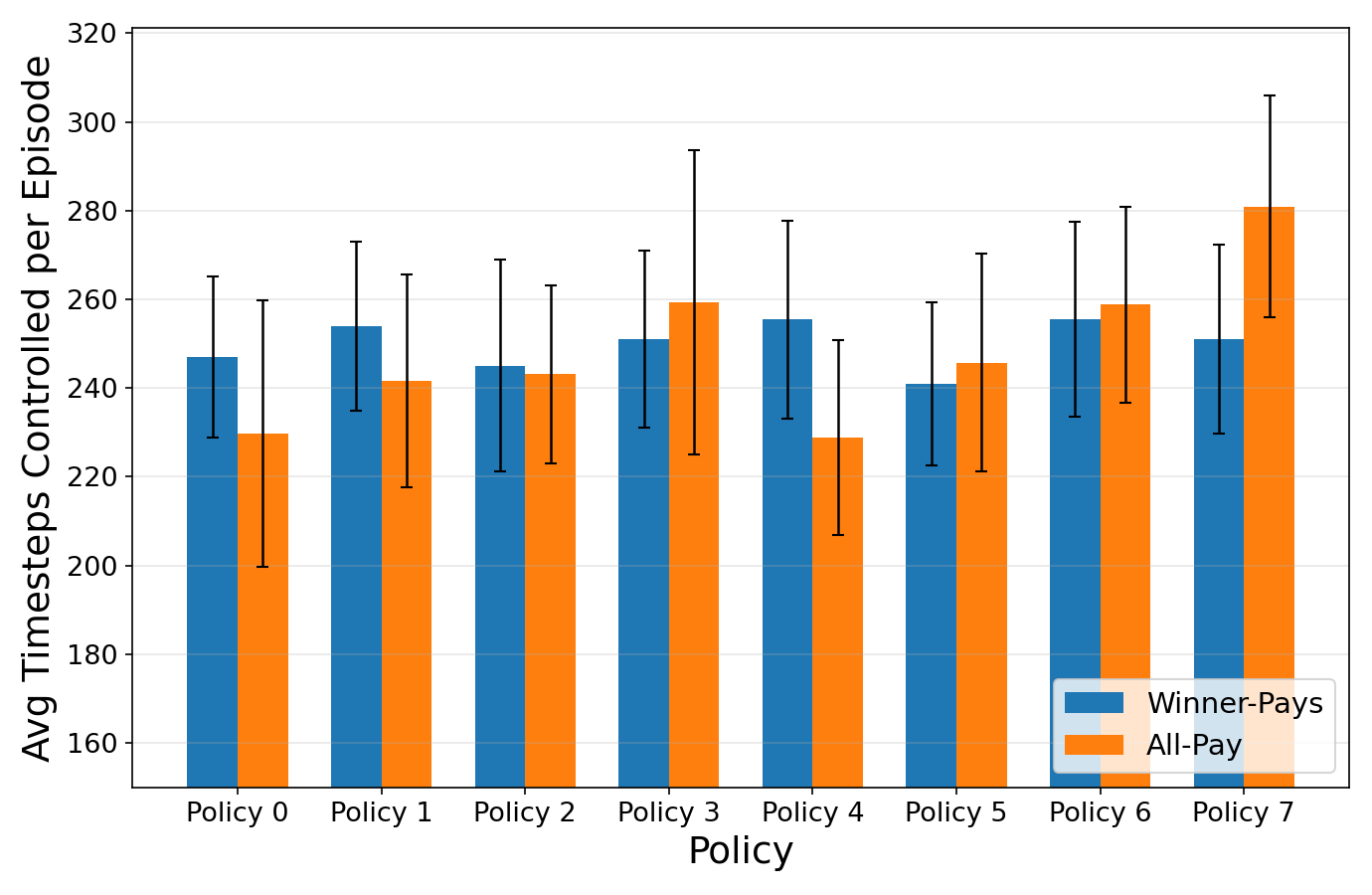}
        \caption{Cat Feeder.}
        \label{fig:dyntar-control-dist}
    \end{subfigure}
    \hfill
    \begin{subfigure}[t]{0.25\textwidth}
        \centering
        \includegraphics[width=\textwidth]{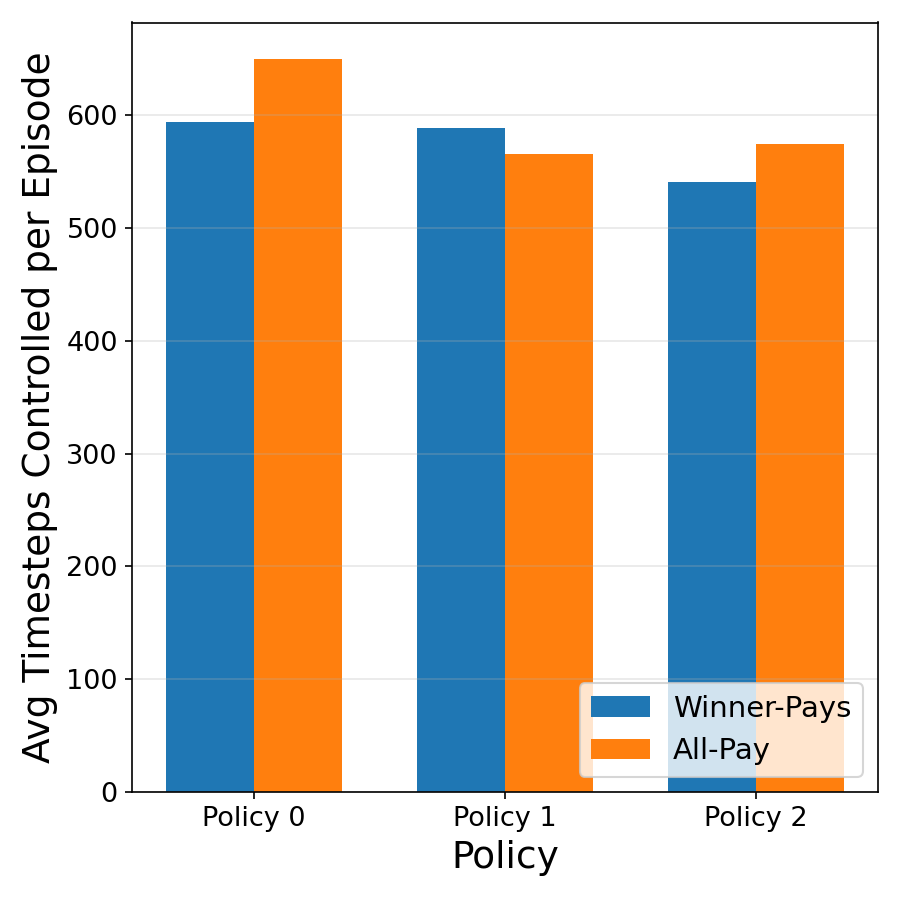}
        \caption{Atari Assault.}
        \label{fig:assault-control-dist}
    \end{subfigure}
    \hfill
    \begin{subfigure}[t]{0.25\textwidth}
        \centering
        \includegraphics[width=\textwidth]{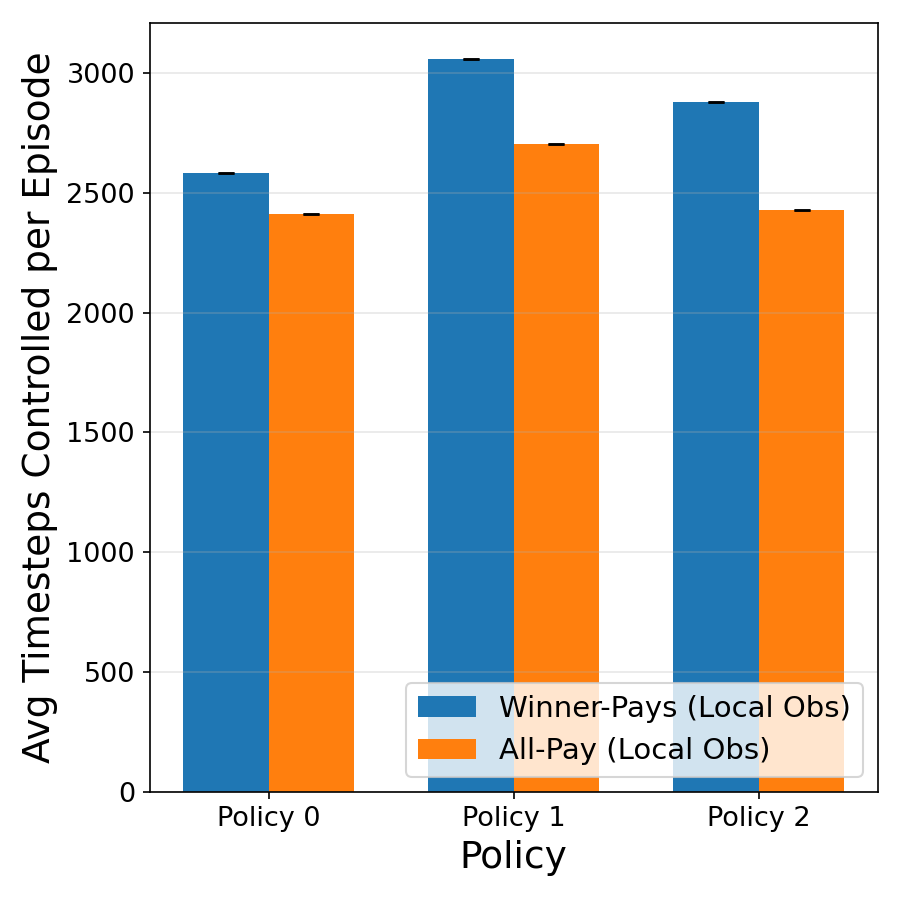}
        \caption{Atari Air-Raid.}
        \label{fig:airraid-control-dist}
    \end{subfigure}
    \caption{Control distributions across environments and methods. The control distribution is the number of timesteps controlled by each policy.}
    \label{fig:control-dist-appendix}
\end{figure}

\begin{figure}[h]
    \centering
    \begin{subfigure}[b]{0.3\textwidth}
        \centering
        \includegraphics[width=\textwidth]{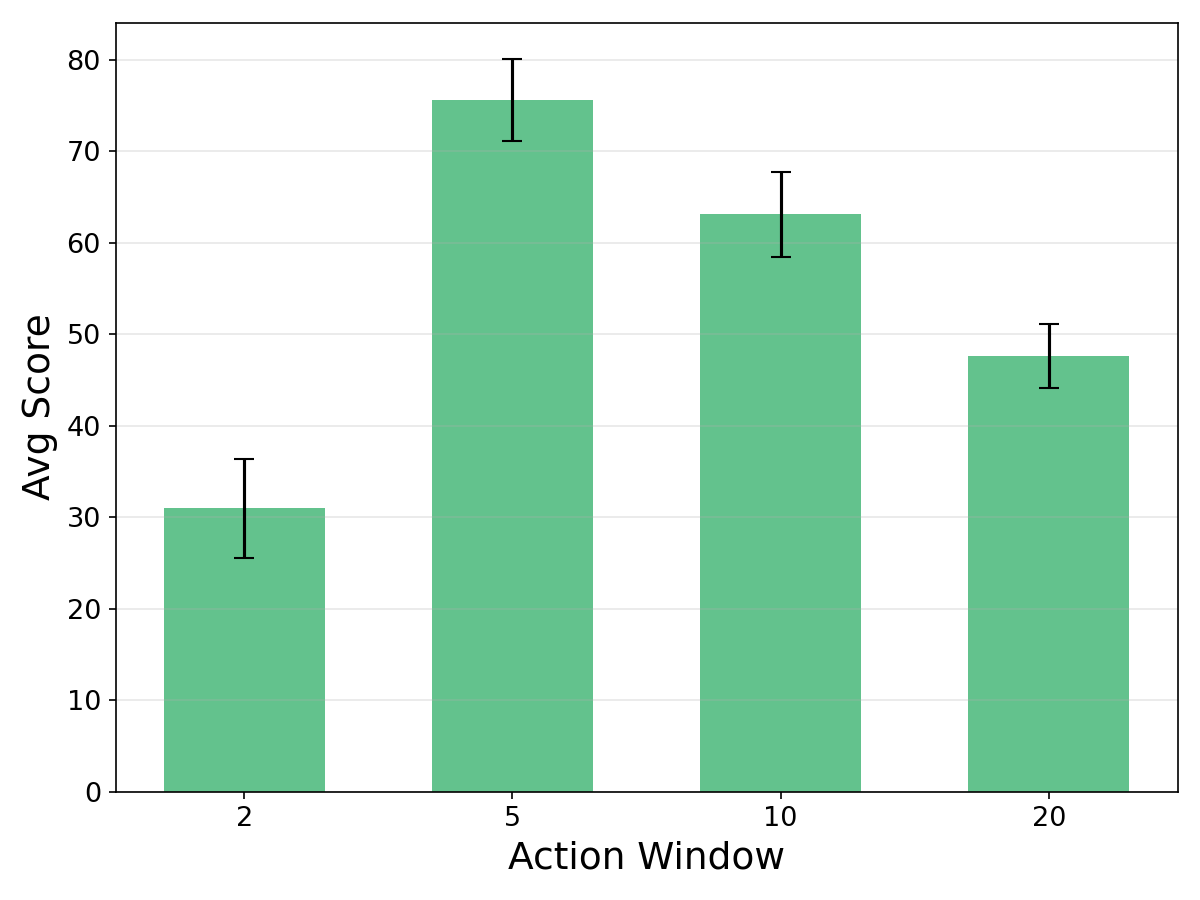}
        \caption{Cat Feeder.}
    \end{subfigure}
    \begin{subfigure}[b]{0.3\textwidth}
        \centering
        \includegraphics[width=\textwidth]{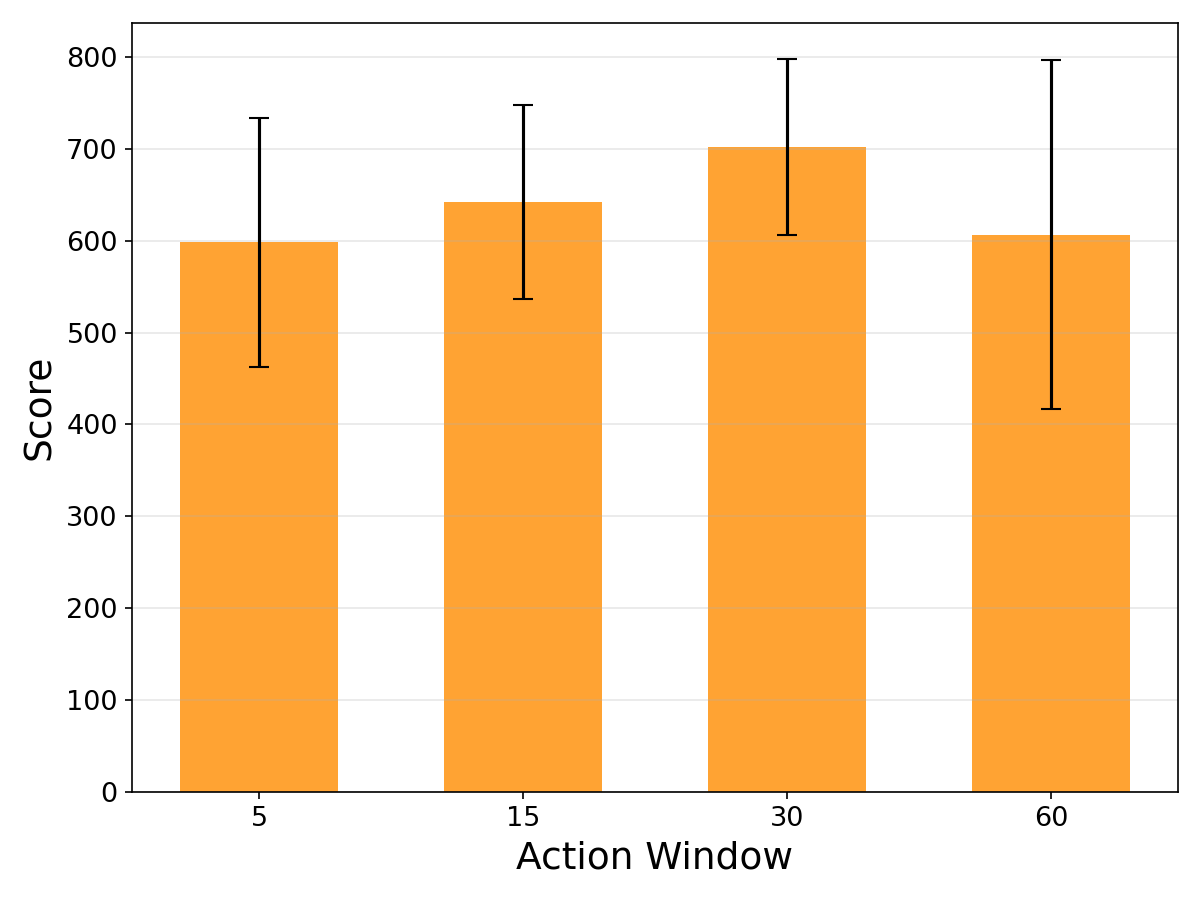}
        \caption{Atari Assault.}
    \end{subfigure}
    \caption{Effect of the action-window size on performance.}
    \label{fig:action-window}
\end{figure}

\Cref{fig:action-window} studies the action-window size in the Cat Feeder and Atari Assault environments.
Smaller windows cause excessive switching between policies, while larger windows prevent sufficient switching to serve all objectives.
The best performance occurs at an intermediate action-window size.

\subsection{Heterogeneous Cat Feeder: Feeding Counts by Priority}
\label{app:heterogeneous-results}
Table~\ref{tab:priority-feeding-counts} breaks down completed feeding requests by priority in the heterogeneous environment.
The priority-weighted score in Table~\ref{tab:heterogeneous-cat-feeder} is the sum of these completion counts weighted by their priorities.
The results show service across all four priority levels while the charging policy competes for control.

\begin{table}[ht]
  \caption{Completed feeding requests by priority: mean (standard deviation). Priority 4 is the highest feeding priority.}
  \label{tab:priority-feeding-counts}
  \centering
  \begin{tabular}{lrrrr}
    \toprule
    Method & Priority 1 & Priority 2 & Priority 3 & Priority 4 \\
    \midrule
    All-Pay & $17.45\;(0.32)$ & $18.75\;(0.70)$ & $19.56\;(0.73)$ & $20.23\;(0.64)$ \\
    Centralized PPO & $6.82\;(1.55)$ & $8.61\;(3.12)$ & $11.72\;(6.13)$ & $13.17\;(7.77)$ \\
    \bottomrule
  \end{tabular}
\end{table}

\subsection{Target Coverage under Scaling}
\label{app:scaling-coverage}
Table~\ref{tab:scaling-coverage} reports the minimum per-target completion count when increasing the number of targets at deployment, averaged over twenty evaluation episodes.
For episode $e$ and target $i$, let $C_{e,i}$ be its completion count.
The reported statistic is
\begin{equation}
  \overline C_{\min}=\frac{1}{E}\sum_{e=1}^{E}\min_i C_{e,i},
  \qquad E=20.
\end{equation}
The minimum is taken within each episode before averaging, rather than taking the minimum of per-target episode averages.
These values are read from the saved JSON field \texttt{statistics.avg\_min\_targets\_reached} and recompute from the twenty episode-level minima.
Scaling selects the single periodic checkpoint with the highest \texttt{avg\_avg\_performance}, as detailed in Appendix~\ref{app:evaluation-aggregation}.
The model is trained with eight targets and transferred by adding copies of the learned policy, each conditioned on its assigned target.
The episode length remains fixed, so additional targets divide the available service time and may expire before being reached.
Positive coverage at every tested size indicates that no target is entirely excluded in this aggregate, but does not imply equal service or invariant performance.

\begin{table}[ht]
  \caption{Scaling Cat Feeder: mean of the per-episode minimum target-completion counts across twenty episodes.}
  \label{tab:scaling-coverage}
  \centering
  \begin{tabular}{lrrrr}
    \toprule
    Number of targets & 8 & 10 & 12 & 14 \\
    \midrule
    Minimum completion count & 8.2 & 5.7 & 5.9 & 4.1 \\
    \bottomrule
  \end{tabular}
\end{table}

%% file: implementation-details-appendix.tex
\section{Additional Implementation Details}
\label{app:additional-implementation-details}

\subsection{Parameter Sharing}
\label{app:parameter-sharing}

In our homogeneous-objective experiments, the objectives are symmetric: each local policy solves the same type of reachability task, but for a different assigned target. We therefore share the actor-critic network parameters across local policies. This reduces the number of learned parameters and lets the shared network receive PPO updates from all local-policy objectives.

Parameter sharing does not make the local policies indistinguishable. For each local policy \(i\), the policy input contains both a pooled embedding of the full set of objective vectors and the objective-specific vector \(z_i\) for that policy's assigned target. Thus two policy copies share network weights, but are conditioned on different \(z_i\)'s and receive different local rewards. In Cat Feeder, \(z_i\) encodes the assigned feeding request, including its position and remaining lifetime; in the Atari environments, \(z_i\) encodes the assigned enemy object exposed by the object-centric observation.

\subsection{Urgency-Weighted Scalarization}
\label{app:urgency-scalarization}
Let $L$ denote the target expiry interval (\texttt{target\_expiry\_steps}), distinct from the episode horizon $H$, and let $c_{i,t}$ be target $i$'s age counter in the state before action $t$.
Define its clamped remaining lifetime, urgency, and normalized weight by
\begin{align}
  \ell_{i,t} &= \max\{L-c_{i,t},1\}, &
  q_{i,t} &= \frac{\mathbf{1}[i\text{ active at }t]}{\ell_{i,t}}, \\
  w_{i,t} &=
  \begin{cases}
    q_{i,t}/\sum_j q_{j,t}, & \sum_j q_{j,t}>0,\\
    0, & \text{otherwise}.
  \end{cases}
\end{align}
The weights are computed before movement, counter increments, expiration, or respawn.
Active weights sum to one, reached static targets receive zero weight, and there is no additional normalization by the number of targets.
For example, two active targets with remaining lifetimes 1 and 4 receive weights $0.8$ and $0.2$.
The scalar reward is
\begin{equation}
  r_t = \sum_i w_{i,t} r^{\mathrm{obj}}_{i,t}
        -\lambda_B\mathbf{1}[\text{battery depletion at }t],
\end{equation}
where $r^{\mathrm{obj}}_{i,t}$ includes distance-improvement shaping, priority-weighted completion reward, and the expiry penalty.
The depletion term applies only in battery-enabled environments and is not multiplied by target urgency weights.

\subsection{Heterogeneous Cat Feeder Implementation}
\label{app:heterogeneous-implementation}
\paragraph{Battery and stations.}
The reported moving-station comparison uses battery capacity 50.
The three stations start at $(0,0)$, $(15,15)$, and $(29,29)$ and follow independent bounded persistent random walks, moving every 15 steps with direction-change probability $0.1$.
A successful one-cell movement consumes one battery unit; boundary or no-movement steps consume none.
Occupying a charging station immediately restores the battery to capacity.
Reaching zero away from a station causes an immediate tow to the nearest station and a full refill; we count such events as battery depletions.
A capacity-40 fixed-station reference configuration is a different benchmark and is not the configuration used for the reported moving-station results.

\paragraph{Effective bids and payments.}
Feeders and charger both choose raw categorical bids in $\{0,\ldots,6\}$.
Let $B_t$ be the current battery level and $e_t$ the energy required to reach the nearest station.
The charger is forced to bid zero unless $B_t\leq e_t+8$.
For a permitted raw charger bid $b_{\mathrm{chg},t}$, the bid used to determine the winner is
\begin{equation}
  b^{\mathrm{eff}}_{\mathrm{chg},t}
  = b_{\mathrm{chg},t}
    +2\mathbf{1}[b_{\mathrm{chg},t}>0\ \land\ B_t\leq25].
\end{equation}
Thus the effective charger ceiling is 8 when the boost applies, versus 6 for feeders.
The All-Pay payment uses the raw bid, not the boosted effective bid.
This battery-dependent priority rule is an implementation extension to the common-cap auction analyzed in the theory.

\paragraph{Parameter sharing.}
The eight feeder controllers share one actor--critic network, with observations reordered according to the assigned target identity.
The charger uses a completely separate actor--critic with separate direction and bid actor trunks in the moving-station experiment.
No parameters are shared between feeders and charger.

\paragraph{Centralized baseline.}
The unassisted centralized PPO baseline in Table~\ref{tab:heterogeneous-cat-feeder} directly chooses body movement without bidding, an action window, or an execution-time action override.
Let $d_{i,t}$ be distance to target $i$, $p_i\in\{1,2,3,4\}$ its feeding priority, and $D_t$ the battery-depletion indicator.
Its unweighted scalar reward is
\begin{equation}
\begin{split}
 r_t={}&\sum_i\bigl[0.6(d_{i,t}-d_{i,t+1})
       +20p_i\mathbf{1}[i\text{ reached at }t]\\
       &\hspace{3em}-50\mathbf{1}[i\text{ expired at }t]\bigr]-10D_t.
\end{split}
\end{equation}
An assisted centralized variant adds reserve-gated progress shaping toward a station, a refill bonus proportional to the restored battery fraction, and an additional depletion penalty of $100$.
It also has no execution-time action override; these auxiliary terms are not part of the unassisted baseline reported here.

\subsection{Evaluation Aggregation}
\label{app:evaluation-aggregation}
For Table~\ref{tab:all_methods_comp}, PPO results first average each seed's last five available evaluation checkpoints, then report the mean and sample standard deviation across seed averages.
The local-observation ablation (\texttt{visible\_targets=0}, called blind in the discussion) uses the same aggregation as the other PPO variants.
Scaling evaluations use a different checkpoint-selection procedure: a single periodic checkpoint is selected by the largest \texttt{avg\_avg\_performance}, rather than averaging the final five checkpoints.
This selection uses evaluation performance and should be distinguished from evaluating a prespecified final checkpoint on an independent test set.
The scaling plot displays eight times the mean per-episode average performance, which is distinct from the minimum-completion statistic defined in Appendix~\ref{app:scaling-coverage}.

%% file: hyperparameters.tex
\section{Reproducibility Details}
\label{sec:reprod-details}

The code for all experiments is included in the supplementary materials.

All experiments use 10 random seeds: \{1825, 410, 4507, 4013, 3658, 5215, 6861, 6803, 7819, 8057\}.

\paragraph{Compute resources.}
All experiments were conducted on a single Ubuntu 22.04.5 LTS machine running Linux kernel 5.15.0-161-generic. The machine has two Intel Xeon Gold 6248 CPUs at 2.50 GHz, providing 40 physical CPU cores and 80 hardware threads, 692 GiB of system memory, and four NVIDIA GeForce RTX 2080 Ti GPUs detected on the PCI bus. The NVIDIA driver and kernel module version was 555.42.02.

\paragraph{Existing assets and licenses.}
Our experiments use the Arcade Learning Environment (ALE)~\cite{bellemare13arcade}, OCAtari~\cite{delfosse2023ocatari}, CleanRL~\cite{huang2022cleanrl}, Optuna~\cite{akiba2019optuna}, and the Deep W-Learning implementation associated with~\cite{rosero2024multi}. We use these assets only for research experiments and cite their original sources. The corresponding software licenses and terms of use are those distributed with the respective public repositories; we include the exact versions, repository URLs, and license names in the released supplementary code.

\subsection{Cat Feeder Environment}

\begin{table}[h]
\caption{Cat Feeder shared (training) environment parameters.}
\label{tab:cat-feeder-env-params}
\centering
\begin{tabular}{ll}
\toprule
Parameter & Value \\
\midrule
Grid size & $30 \times 30$ \\
Number of agents / targets & 8 \\
Target reward & 50.0 \\
Target expiry steps & 200 \\
Target expiry penalty & 50.0 \\
Max episode steps & 2000 \\
Moving targets & True \\
Direction change probability & 0.1 \\
Target move interval & 5 \\
\bottomrule
\end{tabular}
\end{table}

\begin{table}[h]
\caption{Cat Feeder multi-agent PPO hyperparameters (\allpay and \poorman).}
\label{tab:cat-feeder-mappo-hparams}
\centering
\begin{tabular}{ll}
\toprule
Parameter & Value \\
\midrule
\multicolumn{2}{l}{\textit{Auction}} \\
Bid upper bound & 6 \\
Bid penalty & 0.1 \\
Action window & 5 \\
Distance reward scale & 0.6 \\
\midrule
\multicolumn{2}{l}{\textit{Training}} \\
Iterations & 400 \\
Parallel environments & 4096 \\
Steps per rollout & 256 \\
Minibatches & 256 \\
PPO epochs & 4 \\
Learning rate & $2.5 \times 10^{-4}$ (annealed) \\
Discount $\gamma$ & 0.99 \\
GAE $\lambda$ & 0.95 \\
Clip coefficient & 0.05 \\
Entropy coefficient & 0.03 \\
Value function coefficient & 1.0 \\
Max gradient norm & 0.5 \\
\midrule
\multicolumn{2}{l}{\textit{Network}} \\
Actor hidden layers & $[128, 128, 128, 128]$ \\
Critic hidden layers & $[256, 256, 256, 256]$ \\
Target attention pooling & True \\
Target embedding dim & 64 \\
Target encoder hidden layers & $[64, 64]$ \\
\bottomrule
\end{tabular}
\end{table}

\begin{table}[h]
\caption{Cat Feeder single-agent PPO baseline hyperparameters. Hyperparameters were tuned using Optuna with 40 trials. The nearest-target shaping variant is identical but uses a distance reward scale of 0.6 with shaping toward the nearest target.}
\label{tab:cat-feeder-ppo-hparams}
\centering
\begin{tabular}{ll}
\toprule
Parameter & Value \\
\midrule
\multicolumn{2}{l}{\textit{Training}} \\
Iterations & 400 \\
Parallel environments & 4096 \\
Steps per rollout & 256 \\
Minibatches & 512 \\
PPO epochs & 8 \\
Learning rate & $1.74 \times 10^{-4}$ (annealed) \\
Discount $\gamma$ & 0.963 \\
GAE $\lambda$ & 0.970 \\
Clip coefficient & 0.327 \\
Entropy coefficient & $1.03 \times 10^{-4}$ \\
Value function coefficient & 1.076 \\
Max gradient norm & 0.840 \\
Distance reward scale & 0.0 \\
Target attention pooling & False \\
\midrule
\multicolumn{2}{l}{\textit{Network (shared with multi-agent PPO)}} \\
Actor hidden layers & $[128, 128, 128, 128]$ \\
Critic hidden layers & $[256, 256, 256, 256]$ \\
\bottomrule
\end{tabular}
\end{table}

\begin{table}[h]
\caption{Cat Feeder DWN baseline hyperparameters.}
\label{tab:cat-feeder-dwn-hparams}
\centering
\begin{tabular}{ll}
\toprule
Parameter & Value \\
\midrule
\multicolumn{2}{l}{\textit{Training}} \\
Total timesteps & $5 \times 10^{8}$ \\
Parallel environments & 256 \\
Discount $\gamma$ & 0.99 \\
Replay buffer size & $1 \times 10^{6}$ \\
Batch size & 256 \\
Learning starts & $1 \times 10^{5}$ \\
Train frequency & 256 \\
W-network train delay & $1 \times 10^{6}$ \\
Target network update frequency & 1000 \\
Target network $\tau$ & 1.0 \\
\midrule
\multicolumn{2}{l}{\textit{Networks}} \\
Q-network hidden layers & $[256, 256, 256, 256]$ \\
W-network hidden layers & $[128, 128, 128]$ \\
Q-network learning rate & $1 \times 10^{-4}$ \\
W-network learning rate & $1 \times 10^{-4}$ \\
\midrule
\multicolumn{2}{l}{\textit{Epsilon schedules}} \\
Q $\varepsilon$ start / min / decay & 0.99 / 0.01 / 0.99 \\
W $\varepsilon$ start / min / decay & 0.99 / 0.01 / 0.99 \\
\bottomrule
\end{tabular}
\end{table}

\subsection{Assault Environment}

\begin{table}[h]
\caption{Assault shared environment parameters.}
\label{tab:assault-env-params}
\centering
\begin{tabular}{ll}
\toprule
Parameter & Value \\
\midrule
Number of agents & 3 \\
Max enemies & 3 \\
Max episode steps & 10000 \\
Enemy destroy reward & 10.0 \\
Life loss penalty & 10.0 \\
Fire-while-hot penalty & 8.0 \\
Action window & 15 \\
Bid upper bound & 2 \\
Bid penalty & 0.3 \\
\bottomrule
\end{tabular}
\end{table}

\begin{table}[h]
\caption{Assault multi-agent PPO hyperparameters (\allpay and \poorman).}
\label{tab:assault-mappo-hparams}
\centering
\begin{tabular}{ll}
\toprule
Parameter & Value \\
\midrule
\multicolumn{2}{l}{\textit{Training}} \\
Iterations & 150 \\
Parallel environments & 128 \\
Steps per rollout & 512 \\
Minibatches & 8 \\
PPO epochs & 8 \\
Learning rate & $1 \times 10^{-4}$ (annealed) \\
Discount $\gamma$ & 0.99 \\
GAE $\lambda$ & 0.95 \\
Clip coefficient & 0.05 \\
Entropy coefficient & 0.05 \\
Value function coefficient & 0.5 \\
Max gradient norm & 0.5 \\
Target KL & 0.02 \\
\midrule
\multicolumn{2}{l}{\textit{Network}} \\
Actor hidden layers & $[128, 128, 128, 128]$ \\
Critic hidden layers & $[256, 256, 256, 256]$ \\
\bottomrule
\end{tabular}
\end{table}

\begin{table}[h]
\caption{Assault single-agent PPO baseline hyperparameters. Entries list differences from the multi-agent PPO setting; all omitted parameters are shared. Target KL early stopping is disabled for this baseline.}
\label{tab:assault-ppo-hparams}
\centering
\begin{tabular}{ll}
\toprule
Parameter & Value \\
\midrule
Learning rate & $2.5 \times 10^{-4}$ (annealed) \\
Clip coefficient & 0.1 \\
Entropy coefficient & 0.01 \\
Target KL & --- \\
\bottomrule
\end{tabular}
\end{table}

\begin{table}[h]
\caption{Assault DWN baseline hyperparameters.}
\label{tab:assault-dwn-hparams}
\centering
\begin{tabular}{ll}
\toprule
Parameter & Value \\
\midrule
\multicolumn{2}{l}{\textit{Training}} \\
Total timesteps & $1 \times 10^{7}$ \\
Parallel environments & 8 \\
Discount $\gamma$ & 0.99 \\
Replay buffer size & $5 \times 10^{5}$ \\
Batch size & 256 \\
Learning starts & $1 \times 10^{4}$ \\
Train frequency & 80 \\
W-network train delay & $1 \times 10^{5}$ \\
Target network update frequency & 1000 \\
Target network $\tau$ & 1.0 \\
\midrule
\multicolumn{2}{l}{\textit{Networks}} \\
Q-network hidden layers & $[256, 256, 256]$ \\
W-network hidden layers & $[128, 128, 128]$ \\
Q-network learning rate & $1 \times 10^{-4}$ \\
W-network learning rate & $1 \times 10^{-4}$ \\
\midrule
\multicolumn{2}{l}{\textit{Epsilon schedules}} \\
Q $\varepsilon$ start / min / decay & 0.99 / 0.01 / 0.995 \\
W $\varepsilon$ start / min / decay & 0.99 / 0.01 / 0.995 \\
\bottomrule
\end{tabular}
\end{table}

%% file: checklist.tex
\section*{NeurIPS Paper Checklist}

\begin{enumerate}

\item {\bf Claims}
    \item[] Question: Do the main claims made in the abstract and introduction accurately reflect the paper's contributions and scope?
    \item[] Answer: \answerYes{}
    \item[] Justification: The abstract and introduction state the paper's main contribution: an auction-based framework for adapting among local policies in multi-objective RL. The experimental claims are supported in Section~\ref{sec:experimental-results}, with scope and future directions discussed in the conclusion.

\item {\bf Limitations}
    \item[] Question: Does the paper discuss the limitations of the work performed by the authors?
    \item[] Answer: \answerYes{}
    \item[] Justification: The conclusion discusses current scope limitations, including explicit objective decomposition, the scope of shared-policy transfer, limited experimental breadth, and the absence of safety guarantees in the heterogeneous feeding-and-charging experiment.

\item {\bf Theory assumptions and proofs}
    \item[] Question: For each theoretical result, does the paper provide the full set of assumptions and a complete (and correct) proof?
    \item[] Answer: \answerYes{}
    \item[] Justification: The theoretical results in Section~\ref{sec:auction-framework} state the assumptions for the bidding settings and include proofs in the main text.

    \item {\bf Experimental result reproducibility}
    \item[] Question: Does the paper fully disclose all the information needed to reproduce the main experimental results of the paper to the extent that it affects the main claims and/or conclusions of the paper (regardless of whether the code and data are provided or not)?
    \item[] Answer: \answerYes{}
    \item[] Justification: Section~\ref{sec:experimental-results} describes the environments, baselines, ablations, and evaluation protocol, while Appendix~\ref{sec:reprod-details} provides seeds and hyperparameters needed to reproduce the main experiments.

\item {\bf Open access to data and code}
    \item[] Question: Does the paper provide open access to the data and code, with sufficient instructions to faithfully reproduce the main experimental results, as described in supplemental material?
    \item[] Answer: \answerYes{}
    \item[] Justification: The supplementary material includes the codebase and recorded rollouts, and Appendix~\ref{sec:reprod-details} reports the seeds and experiment hyperparameters.

\item {\bf Experimental setting/details}
    \item[] Question: Does the paper specify all the training and test details (e.g., data splits, hyperparameters, how they were chosen, type of optimizer) necessary to understand the results?
    \item[] Answer: \answerYes{}
    \item[] Justification: Section~\ref{sec:experimental-results} summarizes the experimental setup, environments, baselines, and evaluation metrics. Appendix~\ref{sec:reprod-details} lists the environment parameters, PPO/DWN hyperparameters, random seeds, and tuning details.

\item {\bf Experiment statistical significance}
    \item[] Question: Does the paper report error bars suitably and correctly defined or other appropriate information about the statistical significance of the experiments?
    \item[] Answer: \answerYes{}
    \item[] Justification: The learning curves in Figure~\ref{fig:learning-curves} report the mean over 10 random seeds with shaded regions corresponding to one standard deviation.

\item {\bf Experiments compute resources}
    \item[] Question: For each experiment, does the paper provide sufficient information on the computer resources (type of compute workers, memory, time of execution) needed to reproduce the experiments?
    \item[] Answer: \answerNo{}
    \item[] Justification: The paper reports the hardware configuration, training lengths, rollout sizes, numbers of parallel environments, and other computational parameters in Appendix~\ref{sec:reprod-details}. Wall-clock execution time and total compute budget were not recorded.
    
\item {\bf Code of ethics}
    \item[] Question: Does the research conducted in the paper conform, in every respect, with the NeurIPS Code of Ethics \url{https://neurips.cc/public/EthicsGuidelines}?
    \item[] Answer: \answerYes{}
    \item[] Justification: The work uses simulated reinforcement-learning environments and does not involve human subjects, private data, or deployed systems. We are not aware of any aspect of the research that conflicts with the NeurIPS Code of Ethics.

\item {\bf Broader impacts}
    \item[] Question: Does the paper discuss both potential positive societal impacts and negative societal impacts of the work performed?
    \item[] Answer: \answerNA{}
    \item[] Justification: The paper is primarily methodological and experimental, and no deployment or human-facing application is studied.
    
\item {\bf Safeguards}
    \item[] Question: Does the paper describe safeguards that have been put in place for responsible release of data or models that have a high risk for misuse (e.g., pre-trained language models, image generators, or scraped datasets)?
    \item[] Answer: \answerNA{}
    \item[] Justification: The paper does not release high-risk generative models, scraped datasets, or other assets with an apparent high risk of misuse.

\item {\bf Licenses for existing assets}
    \item[] Question: Are the creators or original owners of assets (e.g., code, data, models), used in the paper, properly credited and are the license and terms of use explicitly mentioned and properly respected?
    \item[] Answer: \answerYes{}
    \item[] Justification: Appendix~\ref{sec:reprod-details} credits the existing assets used in the experiments, cites their original sources, and states that the corresponding licenses, terms of use, versions, and repository URLs are included with the released supplementary code.

\item {\bf New assets}
    \item[] Question: Are new assets introduced in the paper well documented and is the documentation provided alongside the assets?
    \item[] Answer: \answerYes{}
    \item[] Justification: The supplementary material includes the experimental codebase and recorded policy rollouts. The paper documents the environments, training settings, random seeds, and hyperparameters needed to interpret and use these assets.

\item {\bf Crowdsourcing and research with human subjects}
    \item[] Question: For crowdsourcing experiments and research with human subjects, does the paper include the full text of instructions given to participants and screenshots, if applicable, as well as details about compensation (if any)? 
    \item[] Answer: \answerNA{}
    \item[] Justification: The paper does not involve crowdsourcing or research with human subjects.

\item {\bf Institutional review board (IRB) approvals or equivalent for research with human subjects}
    \item[] Question: Does the paper describe potential risks incurred by study participants, whether such risks were disclosed to the subjects, and whether Institutional Review Board (IRB) approvals (or an equivalent approval/review based on the requirements of your country or institution) were obtained?
    \item[] Answer: \answerNA{}
    \item[] Justification: The paper does not involve crowdsourcing or research with human subjects, so IRB approval or equivalent review is not applicable.

\item {\bf Declaration of LLM usage}
    \item[] Question: Does the paper describe the usage of LLMs if it is an important, original, or non-standard component of the core methods in this research? Note that if the LLM is used only for writing, editing, or formatting purposes and does \emph{not} impact the core methodology, scientific rigor, or originality of the research, declaration is not required.
    \item[] Answer: \answerNA{}
    \item[] Justification: The core method does not use LLMs as an important, original, or non-standard component.

\end{enumerate}